\pdfoutput=1

\documentclass[11pt]{article}

\usepackage[utf8]{inputenc}
\usepackage[T1]{fontenc}
\usepackage[margin=1in]{geometry}
\usepackage{amsmath,amssymb,amsthm,mathtools,bm}
\usepackage{enumitem}
\usepackage{graphicx}
\usepackage{xcolor}
\usepackage[hidelinks]{hyperref}

\newtheorem{theorem}{Theorem}[section]
\newtheorem{lemma}[theorem]{Lemma}
\newtheorem{proposition}[theorem]{Proposition}

\newtheorem{remark}[theorem]{Remark}
\newtheorem{assumption}[theorem]{Assumption}
\newtheorem{conjecture}[theorem]{Conjecture}
\theoremstyle{definition}

\newcommand{\E}{\mathbb{E}}
\newcommand{\Pbb}{\mathbb{P}}
\newcommand{\R}{\mathbb{R}}

\newcommand{\norm}[1]{\left\lVert #1 \right\rVert}
\newcommand{\abs}[1]{\left\lvert #1 \right\rvert}

\newcommand{\W}{\mathcal{W}}
\newcommand{\Z}{\mathcal{Z}}
\newcommand{\X}{\mathcal{X}}
\newcommand{\I}{\mathbb{I}}
\newcommand{\Hc}{\mathcal{H}}
\newcommand{\LCB}{\mathrm{LCB}}
\newcommand{\UCB}{\mathrm{UCB}}
\graphicspath{{figures/}}

\title{What Do We Care About in Bandits with Noncompliance?\\
\large BRACE: Bandits with Recommendations, Abstention, and Certified Effects}
\author{%
Nicol\'as Della Penna\\[0.25em]
Grouplang\\
\href{mailto:nikete@grouplang.ai}{\texttt{nikete@grouplang.ai}}\\
\href{https://orcid.org/0000-0002-9545-5731}{ORCID: 0000-0002-9545-5731}%
}
\date{}

\hypersetup{
  pdftitle={What Do We Care About in Bandits with Noncompliance? BRACE: Bandits with Recommendations, Abstention, and Certified Effects},
  pdfauthor={Nicol\'as Della Penna},
  pdfsubject={Bandits with noncompliance, objective choice, instrumental variables, and sequential inference},
  pdfkeywords={bandits, noncompliance, instrumental variables, recommendation welfare, treatment policy learning, confidence sequences}
}

\begin{document}
\maketitle

\begin{abstract}
Bandits with noncompliance separate the learner's recommendation from the treatment actually delivered, so the learning target itself must be chosen. A platform may care about recommendation welfare in the current mediated workflow, treatment learning for a future direct-control regime, or anytime-valid uncertainty for one of those targets. These objectives need not agree.

We formalize this objective-choice problem, identify the direct-control regime in which recommendation and treatment objectives collapse, and show by example that recommendation welfare can strictly exceed every learner-measurable treatment policy when downstream actors use private information. For finite-context square-IV problems we propose BRACE, a parameter-free phase-doubling algorithm that performs IV inversion only after matrix certification and otherwise returns full-range but honest structural intervals. BRACE delivers simultaneous policy-value validity, fixed-gap identification of the operationally optimal recommendation policy, and fixed-gap identification of the structurally optimal treatment policy under contextual homogeneity and invertibility. We complement the theory with a finite-context empirical benchmark spanning direct control, mediated present-versus-future tradeoffs, weak identification, homogeneity failure, and rectangular overidentification. The experiments show that safety appears as regret on easy problems, as abstention and wide valid intervals under weak identification, as a reason to prefer recommendation welfare under homogeneity failure, and as tighter structural uncertainty when extra instruments are available. For rich contexts, we also derive an orthogonal score whose conditional bias factorizes into compliance-model and outcome-model errors, clarifying what must be stabilized for anytime-valid semiparametric IV inference.
\end{abstract}

\section{Introduction}

The classical stochastic bandit model treats the learner's action as the intervention of interest. In many applications, however, the learner controls only a recommendation channel: a physician recommends a therapy but a patient refuses it, an algorithmic suggestion is overwritten by a clinician, or a platform ranks an option that is then filtered through downstream human discretion. In these settings, the learner chooses a recommendation (or instrument) $Z$, while the realized treatment $X$ is a separate random variable.

That separation changes more than the feedback model. It changes what the experiment should be optimizing. If the existing recommendation channel is the thing that will actually be deployed, then recommendation welfare may be the right target. If the goal is instead to learn the best treatment rule for a future direct-assignment regime, then the relevant object is structural treatment welfare. And if the goal is scientific reporting under adaptive stopping, then valid sequential inference is primary. The central claim of this paper is that these are distinct objectives, not interchangeable metrics for the same one.

This action--treatment split is much older than bandits. In medicine and program evaluation, explicit assignment-versus-receipt problems appear at least in Zelen's prerandomization design~\cite{zelen1979}, Bloom's analysis of no-shows in social experiments~\cite{bloom1984}, and Sommer and Zeger's distinction between programmatic effectiveness and biologic efficacy in randomized trials~\cite{sommer1991}. The modern causal/IV interpretation was then sharpened in the 1990s by Imbens and Angrist~\cite{imbens1994} and Angrist, Imbens, and Rubin~\cite{angrist1996}. To our knowledge, Della Penna, Reid, and Balduzzi~\cite{dellapenna2016} provide the earliest paper in this line devoted specifically to bandits with \emph{observable} noncompliance, introducing compliance-aware protocols, showing that compliance information can either help or hurt learning, and giving hybrid algorithms together with clinically motivated simulations. Kallus~\cite{kallus2018} then formulates \emph{instrument-armed bandits}, makes the causal/IV interpretation explicit, separates regret notions that coincide only under direct control, and develops treatment-oriented algorithms with guarantees. More recently, Oprescu, Cho, and Kallus~\cite{oprescu2025} study adaptive experimentation with noncompliance for a binary-IV average treatment effect target, deriving semiparametric efficiency bounds, a variance-aware adaptive assignment rule, the AMRIV estimator, and anytime-valid asymptotic inference. A detailed literature discussion appears in Section~\ref{sec:related}.

Our starting point is that noncompliance is not one problem but several:
\begin{enumerate}[leftmargin=2em]
    \item \textbf{REC:} maximize welfare under the \emph{current recommendation channel};
    \item \textbf{TRT:} learn the \emph{best treatment rule} for a future regime with direct treatment control;
    \item \textbf{INF:} report \emph{valid uncertainty} for the chosen target under adaptive sampling and stopping.
\end{enumerate}
The correct learning objective should therefore be chosen before the algorithm.

\subsection*{Objective Choice as Regime Choice}

Historically, much of clinical trial methodology has been written in treatment-first language. This is not accidental. In the medical and regulatory settings codified by ICH E9 and E9(R1) \cite{ich1998,ich2021}, the target is typically the effect of a treatment or treatment strategy for future use, so the intervention of scientific interest is close to the thing the trial ultimately hopes to control. In that regime, recommendation, protocol assignment, and treatment are often treated as if they were the same object for decision-making purposes.

Recommendation systems, decision-support tools, and mediated care pathways change that geometry. What gets deployed may be a recommendation rule that is then filtered through clinician judgment, patient choice, moderator discretion, or platform frictions. In such settings, the recommendation channel is often part of the intervention rather than a nuisance to be abstracted away. This is the setting in which REC becomes a first-class objective rather than a proxy for TRT.

Once the two regimes are separated, stakeholder interests need not align. Current participants may prefer REC when the existing channel is the one they will actually face, when downstream discretion uses legitimate private information, or when direct treatment control is unrealistic or undesirable. Future patients, regulators, and system designers may prefer TRT or INF when they want a transportable treatment rule, a channel-independent scientific claim, or evidence that can justify redesigning the workflow itself. Objective choice is therefore a choice about deployment regime, stakeholder horizon, and which forms of mediation are worth preserving, not a technical detail to settle after the algorithm is fixed.

Section~\ref{sec:related} situates this viewpoint relative to the existing literatures on noncompliance, adaptive experimentation, sequential inference, and IV-based policy learning.

The paper makes five contributions.
\begin{enumerate}[leftmargin=2em]
    \item We formalize both ends of the objective-choice problem: a direct-control regime in which REC and TRT collapse, explaining the classical treatment-first norm, and a mediated regime in which the best recommendation policy strictly outperforms every direct treatment policy measurable by the learner.
    \item We propose \textbf{BRACE} (\textbf{b}andits with \textbf{r}ecommendations, \textbf{a}bstention, and \textbf{c}ertified \textbf{e}ffects), a parameter-free objective-first algorithm. It uses phase doubling, explores uniformly until separation, and refuses unstable IV inversion via a matrix-certification fallback.
    \item In a finite-context setting we prove simultaneous policy-value validity, fixed-gap identification of the operational optimum, and fixed-gap identification of the structural optimum.
    \item We complement the theory with an empirical benchmark built around purpose-designed finite-context environments. The study spans the full objective-choice arc from direct-control equivalence to mediated present-versus-future tradeoffs, together with easy identification, weak identification, homogeneity failure, and overidentified rectangular cases.
    \item For rich contexts we derive a candidate orthogonal score with an exact product-form bias identity, making precise both the promise and the weak-identification difficulty for anytime-valid semiparametric IV inference.
\end{enumerate}

This empirical result is not separate from the theory; it is where the conceptual and technical sides of the paper meet. REC does not require the structural homogeneity condition used for TRT point identification, so when homogeneity is doubtful the operational recommendation target remains learnable even as structurally ambitious claims should be weakened or suspended. Section~\ref{sec:empirical} returns to this boundary directly.

\section{Related Literature}
\label{sec:related}

This paper lies at the intersection of several literatures that are close in notation but different in objective. That distinction is central to our argument.

\paragraph{Pre-bandit roots of the assignment--receipt split.}
The underlying modeling issue predates online learning by decades. In clinical trial design, Zelen~\cite{zelen1979} explicitly randomized patients before treatment uptake was resolved, thereby making assignment and receipt distinct objects in the design itself. In social-program evaluation, Bloom~\cite{bloom1984} studied random-assignment experiments with ``no-shows'' and derived estimators for effects when assignment and participation diverge. In clinical-trial analysis, Sommer and Zeger~\cite{sommer1991} distinguished the intent-to-treat or programmatic effect of assignment from the efficacy of treatment actually received. The modern IV/noncompliance formalization then emerged in the 1990s: Imbens and Angrist~\cite{imbens1994} gave the local average treatment effect interpretation, and Angrist, Imbens, and Rubin~\cite{angrist1996} embedded the IV estimand in the Rubin causal model. Our contribution does not claim to invent this split. The novelty is to make it the organizing principle for online bandit learning and sequential inference.

\paragraph{Bandits with observable noncompliance.}
Della Penna, Reid, and Balduzzi~\cite{dellapenna2016} study bandits in which the learner observes both the intended action and the treatment actually taken. In a broader historical search we did not find an earlier bandit paper that made this chosen-action/received-treatment split an explicit online learning object, so we treat their paper as the earliest dedicated bandit treatment of observable noncompliance. Their contribution is not only to point out the problem, but to organize it: they introduce compliance-aware update protocols, show that the extra treatment signal can genuinely help when used carefully but can also destroy sublinear regret when used naively, and propose hybrid algorithms that preserve regret guarantees up to multiplicative factors while still exploiting the extra observations. They also anchor the problem in clinical-trial practice through International Stroke Trial simulations. Our REC objective is closest to this line. The difference is that we make the target of welfare explicit and separate operational recommendation learning from structural treatment learning.

\paragraph{IV bandits and causal regret.}
Kallus~\cite{kallus2018} recasts noncompliance as an instrument-armed bandit problem and makes the causal structure of the setting explicit. A major contribution of that paper is conceptual: it cleanly separates intent-to-treat, treatment-oriented, and complier-style regret notions and shows that they need not agree once instruments and treatments diverge. It also makes an important algorithmic point by proving that standard MAB algorithms cannot generally achieve sublinear treatment regret, and then developing IAB-specific algorithms with logarithmic guarantees under homogeneity-style assumptions. Our TRT objective adopts the same structural perspective, but our emphasis is different. We treat the disagreement between REC and TRT as the first-order modeling choice, and in finite contexts we use parameter-free exploration together with certification and safe fallback instead of tuned exploration.

\paragraph{Strategic or induced compliance.}
Ngo, Stapleton, Syrgkanis, and Wu~\cite{ngo2021} study a game-theoretic model in which the planner uses randomized recommendations to induce compliance over time. This is complementary to our setup. We take the recommendation channel and compliance behavior as given and ask what should be optimized \emph{given} that channel, rather than how to redesign the channel to increase compliance.

\paragraph{Adaptive experimentation and anytime-valid inference.}
A separate literature studies valid inference under adaptive assignment. Kato et al.~\cite{kato2020} analyze adaptive experimental design for average treatment effects; Kato, McAlinn, and Yasui~\cite{kato2021} develop adaptive doubly robust estimation for adaptively collected data; Cook, Mishler, and Ramdas~\cite{cook2024} sharpen semiparametric theory and confidence sequences for adaptive experiments; Liang and Bojinov~\cite{liang2023} propose a design that enables anytime-valid causal inference for bandit experiments; Karampatziakis, Mineiro, and Ramdas~\cite{karampatziakis2021} and Waudby-Smith et al.~\cite{waudby2024} develop confidence-sequence methods for off-policy evaluation in contextual bandits; and Dalal et al.~\cite{dalal2024} give a general route from double/debiased machine learning to anytime-valid inference, including a noncompliance example. These papers are closest in spirit to our INF goal, but they do not study the contextual recommendation--treatment split together with weak-ID-safe structural learning.

\paragraph{IV-based policy learning and confounded decision making.}
Recent offline works use instruments or other side information to learn policies from confounded data. Chen et al.~\cite{chen2023} study offline contextual bandits with confounding bias and missing observations; Shao et al.~\cite{shao2024} develop double-machine-learning IV methods for offline decision learning; and Oprescu and Kallus~\cite{oprescu2024} combine weak-IV and observational data for heterogeneous treatment-effect estimation. These works are adjacent to our structural target and rich-context roadmap, but they are offline or batch settings rather than online contextual noncompliance with sequential validity and stopping.

\paragraph{Recent adaptive-IV work under noncompliance.}
The paper closest to ours in causal flavor is Oprescu, Cho, and Kallus~\cite{oprescu2025}, which studies adaptive experimentation with noncompliance for a binary instrumental variable and an average treatment effect target. Their contribution is strongest on the semiparametric side: they derive the adaptive-assignment efficiency bound for this IV problem, identify the variance-aware optimal assignment rule, introduce the adaptive multiply-robust AMRIV estimator, and establish asymptotic normality, rates, and anytime-valid asymptotic confidence sequences. Our finite-context results are less semiparametrically ambitious, but broader along a different axis: a multicategory contextual policy-learning problem with an explicit separation between operational recommendation welfare and structural treatment learning.

\paragraph{Terminology and objective choice.}
The phrase \emph{non-compliant bandits} is also used by Kveton et al.~\cite{kveton2023} for bandits whose proposed actions may be overwritten by a downstream system to satisfy another task. That notion is operationally similar to action overwrite, but it is not an IV/potential-outcomes formulation and does not study structural treatment effects. Finally, our emphasis on distinguishing welfare during the experiment from the quality of the policy learned after the experiment is philosophically aligned with Athey et al.~\cite{athey2022}, who contrast within-experiment outcomes with policy learning in a contextual bandit field experiment. Our contribution is to make this separation explicit in the specific setting of noncompliance, where recommendation-level and treatment-level targets are distinct mathematical objects even before one discusses exploration.

\section{The Three Objectives and the Operational Advantage}
\label{sec:objectives}

We begin with a general formulation that keeps recommendations and treatments conceptually distinct. Let $\Z$ denote the recommendation (instrument) action space and let $\X=[K]$ denote the treatment space. At round $t$, context $W_t\in\W$ is observed. The incoming unit carries
\begin{itemize}
    \item potential rewards $\{Y_t(x):x\in\X\}$, and
    \item a compliance type $C_t:\Z\to\X$ mapping each recommendation to the treatment that would actually be delivered.
\end{itemize}
The learner chooses $Z_t\in\Z$ after observing $W_t$ but before seeing the unit's latent variables. The realized treatment is
\[
X_t=C_t(Z_t),
\]
and the observed reward is
\[
Y_t=Y_t(X_t)\in[0,1].
\]

A recommendation policy is a map $\pi_{\mathrm{rec}}:\W\to\Z$. A structural treatment policy is a map $\pi_{\mathrm{str}}:\W\to\X$. Three natural targets coexist:
\begin{enumerate}[leftmargin=2em]
    \item \textbf{Operational Recommendation Welfare (REC):}
    \[
    V^{\mathrm{rec}}(\pi_{\mathrm{rec}})
    := \E\!\left[Y\!\left(C(\pi_{\mathrm{rec}}(W))\right)\right].
    \]
    This is the value of deploying the recommendation policy in the \emph{current} workflow, including downstream overrides.
    \item \textbf{Structural Treatment Welfare (TRT):}
    \[
    V^{\mathrm{str}}(\pi_{\mathrm{str}})
    := \E\!\left[Y\!\left(\pi_{\mathrm{str}}(W)\right)\right].
    \]
    This asks what would happen in a counterfactual regime with \emph{direct treatment control}.
    \item \textbf{Scientific Inference (INF):} the goal is to maintain a confidence sequence $C_t$ around a chosen target parameter (for example $V^{\mathrm{str}}(\pi)$ for all $\pi$ in a finite class) such that
    \[
    \Pbb\!\bigl(\forall t,\ \theta\in C_t\bigr)\ge 1-\delta.
    \]
\end{enumerate}

In the classical direct-control regime, these objectives coincide.

\begin{proposition}[REC and TRT coincide under direct control]
\label{prop:collapse}
Suppose $\Z=\X$ and $C(z)=z$ almost surely for every $z\in\X$. Then for every policy $\pi:\W\to\X$,
\[
V^{\mathrm{rec}}(\pi)=V^{\mathrm{str}}(\pi).
\]
\end{proposition}

\begin{proof}
If $C(z)=z$ almost surely for every $z$, then $C(\pi(W))=\pi(W)$ almost surely. Therefore
\[
V^{\mathrm{rec}}(\pi)
=
\E\!\left[Y\!\left(C(\pi(W))\right)\right]
=
\E\!\left[Y\!\left(\pi(W)\right)\right]
=
V^{\mathrm{str}}(\pi).
\]
\end{proof}

Proposition~\ref{prop:collapse} helps explain why classical trial language is usually treatment-first. Once recommendations and treatments separate, however, the point of distinguishing REC from TRT is not semantic. The two targets can rank policies differently.

\begin{proposition}[Strict operational advantage under private discretion]
\label{prop:strict}
There exists an environment in which the best recommendation policy strictly outperforms every direct treatment policy measurable by the learner:
\[
\sup_{\pi_{\mathrm{rec}}} V^{\mathrm{rec}}(\pi_{\mathrm{rec}})
>
\sup_{\pi_{\mathrm{str}}} V^{\mathrm{str}}(\pi_{\mathrm{str}}).
\]
\end{proposition}

\begin{proof}
Let $W$ be constant. Let the treatment space be $\X=\{1,2\}$ and the recommendation space be $\Z=\{a,b\}$. A downstream clinician observes a private signal $U\in\{1,2\}$ with $\Pbb(U=1)=\Pbb(U=2)=1/2$, hidden from the learner. Let the potential rewards be
\[
Y(x)=\I\{x=U\},
\]
so the welfare-maximizing treatment equals the private signal.

Let recommendation $a$ trigger individualized clinician discretion and recommendation $b$ trigger a default treatment:
\[
C(a)=U,
\qquad
C(b)=1.
\]
The constant recommendation policy $\pi^\dagger_{\mathrm{rec}}\equiv a$ achieves
\[
V^{\mathrm{rec}}(\pi^\dagger_{\mathrm{rec}})
=
\E\bigl[Y(C(a))\bigr]
=
\E\bigl[Y(U)\bigr]
=1.
\]
By contrast, since $W$ is constant and $U$ is hidden, any direct treatment policy measurable by the learner is constant in $\{1,2\}$. For either constant choice $x$,
\[
V^{\mathrm{str}}(x)=\Pbb(U=x)=1/2.
\]
Hence the best recommendation policy has value $1$, while the best direct treatment policy has value $1/2$.
\end{proof}

\begin{remark}[Why this does not contradict structural learning]
Proposition~\ref{prop:strict} is a welfare comparison, not an IV-identification result. In the example, the same private signal that affects compliance also changes which treatment is best, so the homogeneity condition used later for structural identification generally fails. This is precisely why the objective choice matters: operational delegation can be welfare-superior for present patients even when a transportable structural treatment target is unavailable or requires stronger assumptions.
\end{remark}

\begin{remark}[Choosing between REC and TRT]
Neither REC-first nor TRT-first is a universal default. REC is the right primary objective when the recommendation channel is itself the legitimate deployment target: when future patients will in fact receive recommendations rather than direct treatment assignment, when downstream discretion uses private information that is normatively appropriate to preserve, and when the main goal is to improve realized welfare inside the existing workflow rather than to redesign that workflow. In such settings, trial participants may positively prefer REC-first learning. They may care most about whether the actual system they will face---including clinician judgment, triage discretion, or implementation frictions---works well for people like them, rather than whether the trial identifies an idealized direct-treatment rule that no one can currently implement.

TRT or INF becomes the better primary objective when the scientific or policy aim is transportable treatment guidance, channel redesign, or transparent treatment advice that should survive changes in the recommendation mechanism. Participants may also prefer TRT-first learning when they do not want their outcomes improved by optimizing opaque nudges, gatekeeping, or discretionary overrides, and instead want the trial to learn a treatment rule they could later understand, contest, or access more directly. In that case, the recommendation channel is a nuisance or constraint to work around, not the object one wants to optimize.

The substantive choice is therefore not between a ``practical'' objective and a ``scientific'' one. It is between two different deployment targets. REC treats the current recommendation process as part of the intervention; TRT treats it as an obstacle between the learner and the treatment effect of interest. Historically, many trials have defaulted to treatment-first language, but bandit problems with noncompliance need not inherit that norm. The objective should instead be chosen by asking what future regime is intended and what tradeoff among realized welfare, autonomy, transparency, transportability, and channel reform participants and decision makers are prepared to endorse.
\end{remark}

\section{Finite-Context Model and Identification}
\label{sec:finite}

We now specialize to the square setting $\Z=\X=[K]$, where the recommendation labels and treatment labels share the same index set. This is the regime in which matrix-based IV identification is natural.

Assume the context space is finite, $\W=\{1,\dots,S\}$. At each round, an i.i.d.\ draw
\[
(W_t,C_t,Y_t(1),\dots,Y_t(K))
\]
is generated from a stationary population, with context marginal $\nu(w)>0$ for every $w\in\W$. The learner observes $W_t$, then selects $Z_t\in[K]$ independently of the current unit's latent variables conditional on $W_t$ and past history, and then observes
\[
X_t=C_t(Z_t),
\qquad
Y_t=Y_t(X_t).
\]

For each context $w$, define
\[
P(w)_{zx} := \Pbb(X=x\mid W=w,Z=z)=\Pbb(C(z)=x\mid W=w),
\]
\[
g_z(w) := \E[Y\mid W=w,Z=z],
\qquad
\mu_x(w) := \E[Y(x)\mid W=w].
\]
Thus $P(w)$ is the context-specific compliance matrix, $g(w)$ is the vector of operational intent-to-treat means, and $\mu(w)$ is the vector of structural treatment means.

\begin{assumption}[Contextual homogeneity]\label{asm:hom}
For every $w\in\W$ and $x,x'\in[K]$,
\[
\E[Y(x)-Y(x')\mid W=w,C]
=
\E[Y(x)-Y(x')\mid W=w]
\qquad\text{a.s.}
\]
where $C=(C(1),\dots,C(K))$ denotes the compliance type.
\end{assumption}

Assumption~\ref{asm:hom} is precisely where the REC/TRT divide becomes technically consequential. REC does not require it, whereas TRT point identification does. Section~\ref{sec:empirical} later treats violations of this assumption as a deliberate stress test rather than a mere nuisance.

\begin{assumption}[Contextual identifiability]\label{asm:id}
For every $w\in\W$, the matrix $P(w)$ is invertible.
\end{assumption}

\begin{lemma}[Contextual IV identification]\label{lem:id}
Under Assumption~\ref{asm:hom}, for every $w\in\W$,
\[
g(w)=P(w)\mu(w).
\]
Under Assumption~\ref{asm:id}, this implies
\[
\mu(w)=P(w)^{-1}g(w).
\]
\end{lemma}

\begin{proof}
Fix $w$ and choose an arbitrary reference treatment $x_0\in[K]$. For each $x\in[K]$, write
\[
\tau_{x,x_0}(w):=\E[Y(x)-Y(x_0)\mid W=w].
\]
By homogeneity,
\[
\E[Y(x)\mid W=w,C]
=
\E[Y(x_0)\mid W=w,C]+\tau_{x,x_0}(w)
\qquad\forall x.
\]
Hence for any recommendation $z$,
\begin{align*}
g_z(w)
&=
\E[Y(C(z))\mid W=w]\\
&=
\E\!\left[\E[Y(C(z))\mid W=w,C]\mid W=w\right]\\
&=
\E\!\left[\E[Y(x_0)\mid W=w,C]+\tau_{C(z),x_0}(w)\mid W=w\right]\\
&=
\E[Y(x_0)\mid W=w]
+\sum_{x=1}^K \Pbb(C(z)=x\mid W=w)\,\tau_{x,x_0}(w).
\end{align*}
On the other hand,
\[
\mu_x(w)=\E[Y(x_0)\mid W=w]+\tau_{x,x_0}(w).
\]
Since $\sum_x P(w)_{zx}=1$, we obtain
\[
g_z(w)=\sum_{x=1}^K P(w)_{zx}\mu_x(w).
\]
Stacking over $z$ yields $g(w)=P(w)\mu(w)$. Under Assumption~\ref{asm:id}, invertibility gives $\mu(w)=P(w)^{-1}g(w)$.
\end{proof}

\section{The BRACE Algorithm}
\label{sec:algorithm}

The goal is to avoid unstable IV inversion while retaining parameter-free guarantees. BRACE uses phase doubling, uniform exploration before stopping, and objective-specific intervals.

Let phase $r$ end at time $t_r=2^r$. During exploration, choose
\[
Z_t\sim \mathrm{Uniform}([K]).
\]
At phase $r$, let
\[
\hat\nu_r(w)=\frac1{t_r}\sum_{t\le t_r}\I\{W_t=w\},
\qquad
N_r(w,z)=\sum_{t\le t_r}\I\{W_t=w,\ Z_t=z\}.
\]
When $N_r(w,z)\ge 1$, define the empirical operational mean
\[
\hat g_{r,z}(w)
=
\frac{\sum_{t\le t_r}\I\{W_t=w,\ Z_t=z\}Y_t}{N_r(w,z)}
\]
and the empirical compliance row
\[
\hat P_r(w)_{zx}
=
\frac{\sum_{t\le t_r}\I\{W_t=w,\ Z_t=z,\ X_t=x\}}{N_r(w,z)}.
\]
When $N_r(w,z)=0$, we leave $\hat g_{r,z}(w)$ and $\hat P_r(w)_{z\cdot}$ undefined and use full-range local intervals.

For a confidence split $\delta_r\propto \delta/(r+1)^2$, define radii
\[
a_r(w,z)=\sqrt{\frac{2\log\!\bigl(2^KSK(r+1)^2/\delta\bigr)}{N_r(w,z)\vee 1}},
\qquad
b_r(w,z)=\sqrt{\frac{\log\!\bigl(4SK(r+1)^2/\delta\bigr)}{2(N_r(w,z)\vee 1)}},
\]
and a context-frequency radius
\[
d_r(w)=\sqrt{\frac{\log\!\bigl(4S(r+1)^2/\delta\bigr)}{2t_r}}.
\]
Write
\[
a_r(w)=\max_z a_r(w,z),
\qquad
\eta_r=\sum_{w\in\W} d_r(w).
\]

\paragraph{Matrix certification.}
If $\hat P_r(w)$ is invertible and
\[
\norm{\hat P_r(w)^{-1}}_\infty\, a_r(w)\le \frac12,
\]
we call context $w$ \emph{certified} at phase $r$. In that case define
\[
\hat\mu_r(w)=\hat P_r(w)^{-1}\hat g_r(w),
\qquad
c_r(w)=\norm{\hat P_r(w)^{-1}}_\infty\bigl(a_r(w)+\max_z b_r(w,z)\bigr).
\]

\paragraph{Local intervals.}
For the REC target, define for each context $w$ and action $a\in[K]$
\[
I^{\mathrm{rec}}_{r,w}(a)=
\begin{cases}
[\hat g_{r,a}(w)-b_r(w,a),\ \hat g_{r,a}(w)+b_r(w,a)]\cap[0,1], & N_r(w,a)\ge 1,\\[4pt]
[0,1], & N_r(w,a)=0.
\end{cases}
\]
For the TRT and INF targets, define
\[
I^{\mathrm{str}}_{r,w}(a)=
\begin{cases}
[\hat\mu_{r,a}(w)-c_r(w),\ \hat\mu_{r,a}(w)+c_r(w)]\cap[0,1], & \text{if } w \text{ is certified},\\[4pt]
[0,1], & \text{otherwise}.
\end{cases}
\]

Let $\Pi$ be a finite class of policies mapping $\W$ to $[K]$. Fix an objective $\mathrm{obj}\in\{\mathrm{REC},\mathrm{TRT},\mathrm{INF}\}$. Write the local lower and upper endpoints as
\[
\ell^{\mathrm{obj}}_{r,w}(\pi(w)),
\qquad
u^{\mathrm{obj}}_{r,w}(\pi(w)).
\]
Aggregate them into policy bounds
\[
\LCB_r(\pi)
=
\sum_{w\in\W}\hat\nu_r(w)\,\ell^{\mathrm{obj}}_{r,w}(\pi(w))-\eta_r,
\qquad
\UCB_r(\pi)
=
\sum_{w\in\W}\hat\nu_r(w)\,u^{\mathrm{obj}}_{r,w}(\pi(w))+\eta_r.
\]
(The term $\eta_r$ is unnecessary if $\nu$ is known in advance.)

\begin{center}
\fbox{%
\begin{minipage}{0.95\linewidth}
\textbf{Algorithm 1: BRACE (finite-context version)}\smallskip

\textbf{Input:} objective $\in\{\mathrm{REC},\mathrm{TRT},\mathrm{INF}\}$, confidence $\delta\in(0,1)$; when objective $=\mathrm{INF}$, the inferential target is $V^{\mathrm{str}}(\pi)$.\smallskip

\begin{enumerate}[leftmargin=1.5em]
    \item For phases $r=0,1,2,\dots$ with endpoints $t_r=2^r$, sample $Z_t\sim \mathrm{Uniform}([K])$ until the stopping rule below is triggered.
    \item At phase endpoint $t_r$, compute $\hat\nu_r$, $\hat P_r$, $\hat g_r$, certification indicators, and policy intervals $[\LCB_r(\pi),\UCB_r(\pi)]$ for every $\pi\in\Pi$.
    \item \textbf{If objective = INF:} keep exploring and report the current intervals. Between phase endpoints, hold the most recent interval fixed.
    \item \textbf{If objective = REC:} if some $\hat\pi_r$ satisfies
    \[
    \LCB_r(\hat\pi_r)>\max_{\pi\neq \hat\pi_r}\UCB_r(\pi),
    \]
    commit forever to the recommendation policy $Z_t=\hat\pi_r(W_t)$.
    \item \textbf{If objective = TRT:} if some $\hat\pi_r$ satisfies the same strict separation inequality, stop exploration and output $\hat\pi_r$ as the estimated optimal \emph{treatment} policy for future direct-assignment deployment.
\end{enumerate}
\end{minipage}}
\end{center}

\begin{remark}[Structural output versus operational deployment]
For the TRT objective, the algorithm identifies a treatment policy, not a recommendation policy. In the same recommendation-only environment, deploying the labels of the structural optimum as recommendations need not maximize current welfare. The TRT guarantee is therefore an identification guarantee for future direct-treatment use, not a current operational-welfare guarantee.
\end{remark}

\section{Theoretical Guarantees}
\label{sec:theory}

We first record the joint concentration event.

\begin{lemma}[Simultaneous concentration]\label{lem:conc}
There exists an event $\mathcal{E}$ with $\Pbb(\mathcal{E})\ge 1-\delta$ such that simultaneously for all phases $r$, contexts $w$, and actions $z$:
\begin{enumerate}[leftmargin=1.5em]
    \item $\abs{\hat\nu_r(w)-\nu(w)}\le d_r(w)$;
    \item if $N_r(w,z)\ge 1$, then
    \[
    \norm{\hat P_r(w)_{z\cdot}-P(w)_{z\cdot}}_1\le a_r(w,z),
    \qquad
    \abs{\hat g_{r,z}(w)-g_z(w)}\le b_r(w,z).
    \]
\end{enumerate}
\end{lemma}

\begin{proof}
The context-frequency bound is Hoeffding's inequality, union bounded over $r$ and $w$. The mean bound is again Hoeffding. The row-wise compliance bound follows from multinomial $\ell_1$ concentration, union bounded over $r,w,z$.
\end{proof}

The next statement clarifies the role of certification.

\begin{proposition}[Plug-in inversion and certification]\label{prop:local}
On the good event $\mathcal{E}$, whenever $\hat P_r(w)$ is invertible,
\[
\norm{\hat\mu_r(w)-\mu(w)}_\infty
\le
\norm{\hat P_r(w)^{-1}}_\infty\Bigl(a_r(w)+\max_z b_r(w,z)\Bigr).
\]
If in addition $w$ is certified, i.e.
\[
\norm{\hat P_r(w)^{-1}}_\infty a_r(w)\le \frac12,
\]
then
\[
\norm{P(w)^{-1}}_\infty\le 2\norm{\hat P_r(w)^{-1}}_\infty.
\]
\end{proposition}

\begin{proof}
By Lemma~\ref{lem:id},
\[
g(w)=P(w)\mu(w).
\]
Therefore
\[
\hat\mu_r(w)-\mu(w)
=
\hat P_r(w)^{-1}\Bigl((\hat g_r(w)-g(w))+(P(w)-\hat P_r(w))\mu(w)\Bigr).
\]
Taking $\ell_\infty$ norms and using $\norm{\mu(w)}_\infty\le 1$ yields the first claim.

For the second, note that on $\mathcal{E}$,
\[
\norm{\hat P_r(w)^{-1}(P(w)-\hat P_r(w))}_\infty
\le
\norm{\hat P_r(w)^{-1}}_\infty a_r(w)
\le \frac12.
\]
Hence $I+\hat P_r(w)^{-1}(P(w)-\hat P_r(w))$ is invertible by the Neumann series, and
\[
P(w)=\hat P_r(w)\Bigl(I+\hat P_r(w)^{-1}(P(w)-\hat P_r(w))\Bigr)
\]
is invertible with
\[
P(w)^{-1}
=
\Bigl(I+\hat P_r(w)^{-1}(P(w)-\hat P_r(w))\Bigr)^{-1}\hat P_r(w)^{-1}.
\]
The operator norm of the inverse factor is at most $2$, giving the bound.
\end{proof}

Define the policy gaps, when the corresponding optimum is unique, by
\[
\Delta_{\mathrm{rec}}
=
V^{\mathrm{rec}}(\pi^\star_{\mathrm{rec}})
-
\max_{\pi\neq \pi^\star_{\mathrm{rec}}}V^{\mathrm{rec}}(\pi),
\]
\[
\Delta_{\mathrm{str}}
=
V^{\mathrm{str}}(\pi^\star_{\mathrm{str}})
-
\max_{\pi\neq \pi^\star_{\mathrm{str}}}V^{\mathrm{str}}(\pi).
\]
Let
\[
\nu_{\min}:=\min_{w\in\W}\nu(w),
\qquad
L:=\max_{w\in\W}\norm{P(w)^{-1}}_\infty.
\]

\begin{theorem}[Objective-specific guarantees]\label{thm:main}
With probability at least $1-\delta$, all of the following hold.
\begin{enumerate}[leftmargin=1.5em]
    \item \textbf{Inference validity (INF).} Suppose the INF objective is selected, so the inferential target is $V^{\mathrm{str}}(\pi)$. Then for every phase $r$ and every $\pi\in\Pi$,
    \[
    V^{\mathrm{str}}(\pi)\in[\LCB_r(\pi),\UCB_r(\pi)].
    \]
    If the interval displayed at time $t$ is defined to be the most recent phase interval, then coverage holds simultaneously for all $t$. The analogous REC intervals are simpler and valid by the same aggregation argument.

    \item \textbf{Operational identification and regret (REC).} Suppose the REC objective is selected and $\Delta_{\mathrm{rec}}>0$. Then BRACE commits only to the true operational optimum $\pi^\star_{\mathrm{rec}}$. Moreover, up to logarithmic factors in $S,K,|\Pi|,1/\delta$,
    \[
    \tau_{\mathrm{rec}}
    =
    \tilde{\mathcal O}\!\left(\frac{K}{\nu_{\min}\Delta_{\mathrm{rec}}^2}\right)
    \]
    samples suffice for commitment. Since REC incurs zero additional operational regret after correct commitment, its total operational regret up to any horizon $T$ is $\mathcal O(\tau_{\mathrm{rec}})$ on the good event.

    \item \textbf{Structural identification (TRT).} Suppose the TRT objective is selected, $\Delta_{\mathrm{str}}>0$, and Assumption~\ref{asm:id} holds. Then BRACE stops only after identifying the true structural optimum $\pi^\star_{\mathrm{str}}$. Up to logarithmic factors in $S,K,|\Pi|,1/\delta$,
    \[
    \tau_{\mathrm{str}}
    =
    \tilde{\mathcal O}\!\left(\frac{L^2K}{\nu_{\min}\Delta_{\mathrm{str}}^2}\right)
    \]
    samples suffice for identification. This is a treatment-learning guarantee for a future direct-assignment regime; it is not, by itself, a guarantee about welfare under the present recommendation channel.
\end{enumerate}
\end{theorem}

\begin{proof}[Proof sketch]
On the good event $\mathcal{E}$, every local interval is valid. Aggregation with $\hat\nu_r$ and the correction $\eta_r$ preserves validity for policy values, proving Item 1.

For Items 2 and 3, the separation rule cannot commit to a suboptimal policy because that would contradict interval validity. Under uniform exploration, $N_r(w,z)$ grows on the order of $t_r\nu(w)/K$, so the REC widths shrink at rate $\tilde{\mathcal O}(\sqrt{K/(t_r\nu_{\min})})$ and the certified TRT widths shrink at rate $\tilde{\mathcal O}(L\sqrt{K/(t_r\nu_{\min})})$. Once these widths are below half the relevant policy gap, the optimal policy is strictly separated and the stopping rule triggers. The TRT rate uses Proposition~\ref{prop:local} together with the fact that when $a_r(w)\lesssim 1/L$, certification eventually occurs for every context.
\end{proof}

\section{Empirical Study}
\label{sec:empirical}

The theory isolates what BRACE can guarantee in a finite-context square-IV regime. To understand how those guarantees interact with objective choice, misspecification, and natural extensions beyond the theorem statements, we complement the analysis with an executable simulation benchmark. The benchmark includes the paper's core algorithms (BRACE-REC, BRACE-TRT, BRACE-INF), a faster REC/TRT exploration variant, RECERT, a recommendation-first extension that keeps a parallel TRT certificate, representative compliance-aware recommendation baselines following Della Penna et al.~\cite{dellapenna2016}, and IV-oriented treatment baselines following Kallus~\cite{kallus2018}. We also include rectangular and partial-identification extensions to stress the boundaries of the square homogeneous theory rather than to blur them.

All environments are finite-context and small enough that exact REC and TRT policy values can be computed directly. This keeps the empirical study aligned with the paper's finite-context perspective while allowing us to isolate specific difficulties. The benchmark instantiates eleven environments spanning eight substantive case families:
\begin{enumerate}[leftmargin=2em]
    \item \textbf{Direct-control equivalence:} REC and TRT coincide when recommendations are treatments.
    \item \textbf{Easy strong IV:} identification is benign and the main question is efficiency.
    \item \textbf{Operational private-signal advantage:} recommendation welfare can exceed every direct treatment rule available to the learner.
    \item \textbf{Workflow redesign / future control:} present-patient REC under the existing channel can differ sharply from future-regime TRT under direct control.
    \item \textbf{Weak IV with a small structural gap:} unsafe treatment learners can act on noise while weak-ID-safe procedures should abstain.
    \item \textbf{Homogeneity failure:} REC remains well defined but TRT point identification becomes structurally unreliable.
    \item \textbf{Objective mismatch under valid IV assumptions:} the REC-optimal recommendation and TRT-optimal treatment differ even when the IV model is well specified.
    \item \textbf{Rectangular and overidentified cases:} extra recommendation arms can tighten structural information beyond the square setting.
\end{enumerate}

Many of the benchmark's sharpest environments are deliberately one-context problems. Direct-control equivalence, private-signal operational advantage, weak-IV abstention, homogeneity failure, workflow redesign, and the harmful actual-treatment trap all arise without contextual heterogeneity. This is useful rather than limiting: it shows that the paper's main phenomena are not artifacts of context aggregation. Contexts matter for heterogeneous deployment and simultaneous certification across subpopulations, but the REC/TRT split, weak-ID abstention, and homogeneity boundary already appear in the simplest one-context settings.

The two new regime-choice anchors matter for the introduction. In the direct-control environment, every policy has exactly the same REC and TRT value, reproducing the classical treatment-first regime. In the workflow-redesign environment, the best current-channel recommendation attains REC value $0.69$, while the best future direct-treatment rule attains TRT value $0.90$. That is the clean present-versus-future stakeholder split discussed in Section~\ref{sec:objectives}: current participants can rationally prefer the best available recommendation workflow, while future patients or system designers may prefer evidence for channel reform and direct treatment control.

Appendix~\ref{app:empirical} gives additional implementation detail, a fuller scenario-by-scenario discussion, and the rest of the figure suite generated by the benchmark.

Table~\ref{tab:empirical} summarizes the main patterns. The numbers are not meant as a leaderboard. They are there to show \emph{which} part of the objective-first story drives the outcome in each case.

\begin{table}[t]
\centering
\small
\begin{tabular}{p{0.18\textwidth} p{0.47\textwidth} p{0.25\textwidth}}
\hline
\textbf{Scenario} & \textbf{Main lesson} & \textbf{Representative finding} \\
\hline
Direct-control equivalence & When the action--treatment gap disappears, REC and TRT collapse. & The best policy has value $0.775$ under both objectives; ThompsonBounded on REC and 2SLS-epsilon-decay on TRT both recover it. \\
Easy strong IV & Safety is visible mainly as regret when identification is easy. & At the default horizon, BRACE-REC regret is $167.25$, BRACE-REC-FAST is $65.35$, and Actual-UCB is $0.95$. \\
Private-signal operational advantage & REC can be the correct deployable target even when TRT is inferior by construction. & RECERT deploys REC value $1.0$, while its parallel TRT estimate remains $0.5$ and is not deployed. \\
Workflow redesign / future control & Present-channel deployment and future direct-control welfare can point in different directions. & RECERT deploys REC value $0.69$ while estimating TRT value $0.90$ and abstaining on structural deployment. \\
Weak IV, small gap & Under weak identification, the distinctive safe behavior is abstention rather than low regret. & BRACE variants abstain, while unsafe TRT baselines have wrong-non-abstain rates from about $0.10$ to $0.70$. \\
Homogeneity failure & REC remains robust, but TRT should be weakened or withheld. & RECERT keeps deployed REC value $0.95$; 2SLS-style TRT rules deploy the wrong structural policy with value $0.45$. \\
Objective mismatch & Even under valid IV assumptions, operational and structural optima need not agree. & In the tradeoff environment, RECERT deploys REC value $0.85$ while estimating TRT value $0.90$ and abstaining on structural deployment. \\
Rectangular and rescued IV & Extra instruments can materially tighten structural uncertainty. & Partial-ID INF width falls to about $0.18$, versus about $0.36$ in the overidentified case and about $1.07$ in the weak square-IV case. \\
\hline
\end{tabular}
\caption{Representative outcomes from the finite-context simulation benchmark. Figures are drawn from the executable study described in this section, using $10$ seeds per scenario and each scenario's default horizon ($2048$ or $4096$).}
\label{tab:empirical}
\end{table}

Figures~\ref{fig:empirical-core} and~\ref{fig:empirical-intervals} isolate the three empirical comparisons that carry the main argument; the appendix retains the full benchmark-wide plots. The left panel of Figure~\ref{fig:empirical-core} focuses only on the four objective-choice scenarios that matter most for the REC discussion. The right panel focuses only on the representative TRT cases needed to see when structural learners should deploy, abstain, or be regarded as unsafe. Figure~\ref{fig:empirical-intervals} then isolates the design comparison behind the uncertainty story, showing how the same structural target can look very different under strong square identification, weak square identification, and rectangular rescue.

The benchmark is not a pure abstention study. At the default horizons, the base BRACE-TRT rule correctly deploys the structural optimum in all runs of the easy strong-IV and rectangular-overidentified settings, in $9/10$ runs of the workflow-redesign setting, and the rectangular partial-identification variant BRACE-TRT-Partial deploys correctly in all runs of the rescued-IV setting. The more conservative behavior is concentrated in the FAST and RECERT certificate-style variants, and in the genuinely weak or misspecified environments. This matters for interpretation: the experiments do show that the implementation can recover and deploy the right TRT policy when the structure is strong enough; the interesting question is when it should refuse to do so.

\begin{figure}[t]
\centering
\begin{minipage}[t]{0.49\textwidth}
    \centering
    \includegraphics[width=\linewidth]{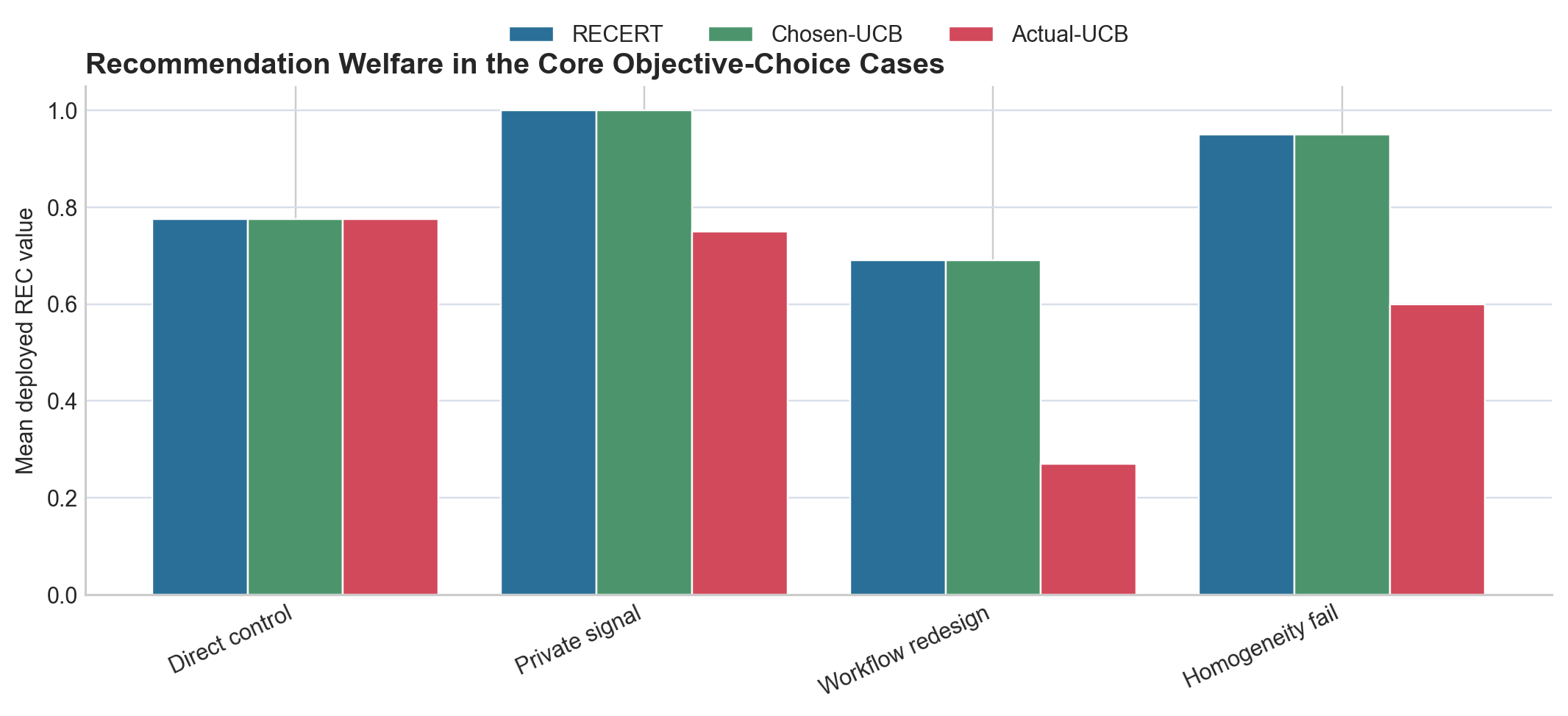}
\end{minipage}
\hfill
\begin{minipage}[t]{0.49\textwidth}
    \centering
    \includegraphics[width=\linewidth]{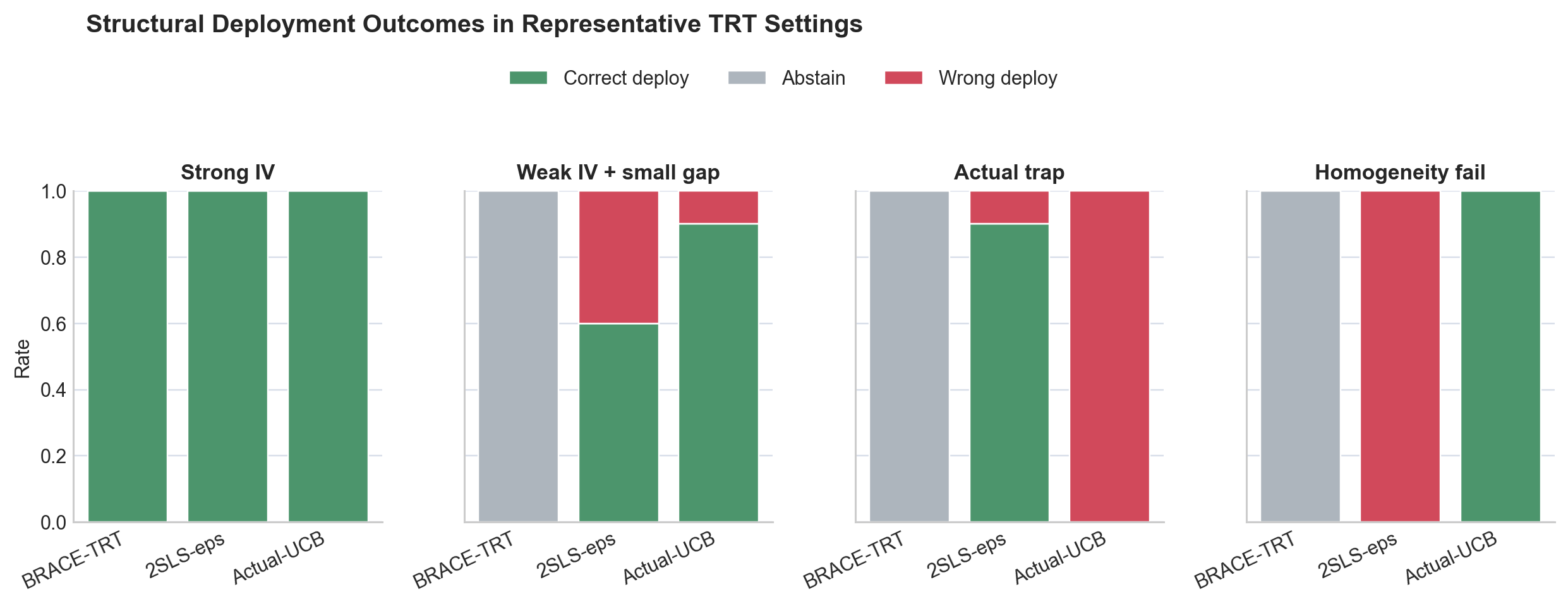}
\end{minipage}
\caption{Core empirical comparisons. Left: in the direct-control, private-signal, workflow-redesign, and homogeneity-failure cases, RECERT matches a strong recommendation baseline while Actual-UCB breaks once recommendation and treatment incentives diverge. Right: in representative TRT settings, the relevant comparison is not just error rate but the split between correct deployment, abstention, and wrong deployment. BRACE abstains in the hard cases, unsafe baselines can fail in different ways, and strong-IV settings still allow correct structural deployment.}
\label{fig:empirical-core}
\end{figure}

\begin{figure}[t]
\centering
\includegraphics[width=0.72\textwidth]{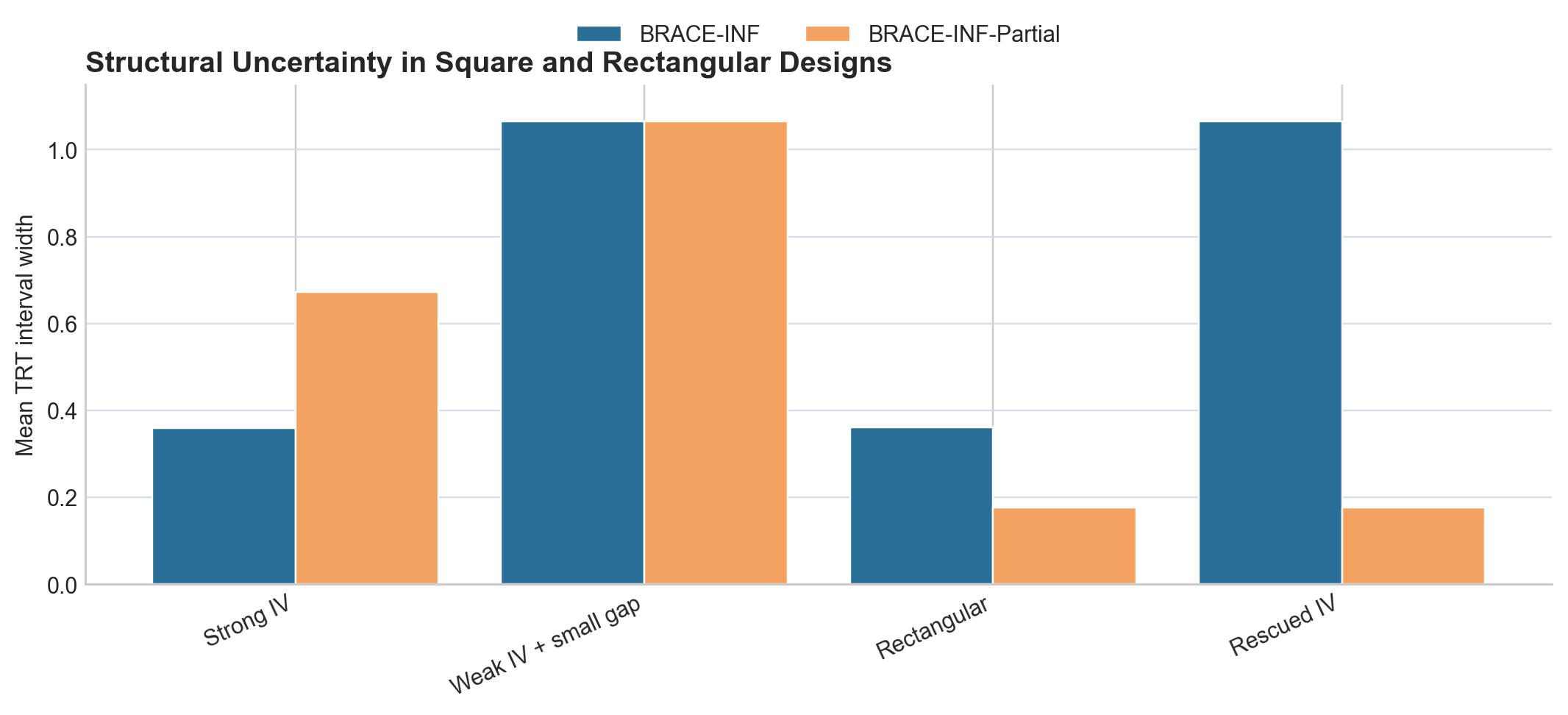}
\caption{Selective structural-uncertainty comparison. Weak square-IV settings produce wide honest intervals, while rectangular overidentification and the rescued-IV design materially tighten them. The appendix shows the full inference panel; this figure isolates the design comparison that matters most for the paper's main claim about certification and extra instruments.}
\label{fig:empirical-intervals}
\end{figure}

Three qualitative conclusions matter for the framing of the paper.

First, the empirical study clarifies both why historical treatment-first language once made sense and why it breaks in mediated settings. In the direct-control equivalence environment, REC and TRT literally coincide. But in the private-signal, tradeoff, workflow-redesign, and homogeneity-violation environments, the two targets separate for different reasons: private information can make REC operationally superior, valid IV structure can still produce different present and future optima, redesign can make future TRT strictly better than current REC, and misspecification can leave REC as the only coherent deployable target.

That said, the benchmark should not be read as an argument that REC-first is the ethically or scientifically correct default for every trial, nor as an argument that TRT should remain the automatic default merely by historical convention. The simulations instead show why the choice must be explicit. When the existing recommendation channel is the real deployment target and participants want the trial to improve outcomes under that channel, REC can be the more decision-relevant object for present use. When the goal is a transportable treatment rule, channel redesign, or a more transparent basis for care than the status quo workflow, TRT or INF may still be the right primary target even when REC looks operationally stronger in the short run.

Second, weak-ID safety should not be judged by regret alone. In easy environments, unsafe baselines can dominate on sample efficiency. The distinctive contribution of BRACE is that, under weak identification, it can widen uncertainty and abstain rather than convert unstable inversion into a treatment claim. The simulations make this concrete: abstention is not a pathology of the benchmark but the behavior the theory is designed to protect.

Third, the extensions beyond the square homogeneous theory help in different ways. Rectangular and partial-identification extensions do not just tighten intervals abstractly; in the rescued-IV setting they can enable correct structural deployment where square point-ID BRACE still abstains. But they do not erase the conceptual boundary between REC and TRT. In particular, when homogeneity fails, the right lesson is not that one should force a sharper TRT estimate; it is that REC may be the only target that remains both well defined and operationally meaningful.

\paragraph{Applied reading of non-certification.}
When BRACE returns only full-range TRT intervals or abstains on structural deployment, that is not merely a negative result; it is a diagnostic. The experiments suggest three different responses. First, if the existing recommendation channel is the real deployment target, then REC or RECERT may already be the right applied answer; the private-signal and homogeneity-failure environments show that one can still learn a useful recommendation policy when a point TRT claim is unavailable or conceptually misaligned. Second, if TRT remains the target, the next question is whether the obstacle is informational rather than conceptual. The rescued-IV experiment shows that design changes can matter: adding a third recommendation arm shrinks structural uncertainty dramatically and can turn abstention into correct deployment. Third, if domain knowledge makes homogeneity doubtful, then a failure to obtain a stable TRT answer should push the analysis toward partial identification, sensitivity analysis, or an explicitly weaker estimand rather than toward more aggressive inversion. In applied use, non-certification should therefore be read as guidance about what to change: the objective, the design, or the strength of the claim.

\section{Rich Contexts: A Semiparametric Roadmap}
\label{sec:rich}

The finite-context analysis avoids nuisance estimation by uniform exploration and empirical averages. With continuous contexts $\W\subseteq\R^d$, structural learning requires nuisance estimation.

Let $\Hc_{t-1}$ be the past history up to time $t-1$. Suppose the exploration propensity $q_t(z\mid W_t)$ is $\Hc_{t-1}$-measurable and satisfies a positivity condition
\[
q_t(z\mid w)\ge q_{\min}>0
\qquad
\forall z,w,t.
\]
Let $\hat P$ and $\hat\mu$ be nuisance estimators trained only on past data (for example by sample splitting or cross-fitting). For a treatment policy $\pi$, define the candidate score
\begin{equation}\label{eq:drscore}
\Gamma_t^\pi(\hat P,\hat\mu)
=
\hat\mu_{\pi(W_t)}(W_t)
+
e_{\pi(W_t)}^\top
\hat P(W_t)^{-1}
\frac{e_{Z_t}}{q_t(Z_t\mid W_t)}
\Bigl(Y_t-\hat\mu_{X_t}(W_t)\Bigr),
\end{equation}
where $e_a$ denotes the $a$th standard basis vector of $\R^K$.

The residual is indexed by the \emph{realized treatment} $X_t$, not the instrument $Z_t$. This is the natural analogue of the IV orthogonal score.

\begin{lemma}[Conditional bias factorization]\label{lem:bias}
Assume $P_0(W_t)$ and $\hat P(W_t)$ are invertible and that $\hat P,\hat\mu$ are measurable with respect to $\Hc_{t-1}$. Then
\[
\E_0\!\left[\Gamma_t^\pi(\hat P,\hat\mu)-\mu_{0,\pi(W_t)}(W_t)\mid W_t,\Hc_{t-1}\right]
=
e_{\pi(W_t)}^\top
\hat P(W_t)^{-1}
\bigl(\hat P(W_t)-P_0(W_t)\bigr)
\bigl(\hat\mu(W_t)-\mu_0(W_t)\bigr).
\]
\end{lemma}

\begin{proof}
Condition on $(W_t,\Hc_{t-1})$. Because $q_t$ is the conditional assignment probability,
\[
\E_0\!\left[\frac{e_{Z_t}}{q_t(Z_t\mid W_t)}\Bigl(Y_t-\hat\mu_{X_t}(W_t)\Bigr)\,\middle|\, W_t,\Hc_{t-1}\right]
=
g_0(W_t)-P_0(W_t)\hat\mu(W_t).
\]
By Lemma~\ref{lem:id} at the population level, $g_0(W_t)=P_0(W_t)\mu_0(W_t)$. Therefore
\[
\E_0[\Gamma_t^\pi(\hat P,\hat\mu)\mid W_t,\Hc_{t-1}]
=
\hat\mu_{\pi(W_t)}(W_t)
+
e_{\pi(W_t)}^\top\hat P(W_t)^{-1}P_0(W_t)\bigl(\mu_0(W_t)-\hat\mu(W_t)\bigr).
\]
Subtracting $\mu_{0,\pi(W_t)}(W_t)$ and rearranging gives
\[
e_{\pi(W_t)}^\top
\Bigl(I-\hat P(W_t)^{-1}P_0(W_t)\Bigr)\bigl(\hat\mu(W_t)-\mu_0(W_t)\bigr),
\]
which equals the stated product-form remainder.
\end{proof}

\begin{remark}[Second-order bias is not the whole story]
The product-form remainder in Lemma~\ref{lem:bias} is encouraging, but it is premultiplied by $\hat P(W_t)^{-1}$. Under weak identification, the inverse can be unstable and amplify nuisance error. Any successful rich-context theory will therefore need a stabilization device analogous to finite-context certification: clipping, regularized inversion, or an explicit fallback region where structural claims are intentionally weakened.
\end{remark}

\begin{conjecture}[Anytime-valid policy inference]
Suppose nuisance estimators are cross-fitted, the product remainder in Lemma~\ref{lem:bias} is $o_p(t^{-1/2})$ uniformly over $\pi\in\Pi$, and the inverse map is stabilized so that predictable line-crossing inequalities remain applicable. Then one should be able to construct confidence sequences for $V^{\mathrm{str}}(\pi)$ that are simultaneously valid over $\pi\in\Pi$ and over all stopping times.
\end{conjecture}

\section{Conclusion}

Contextual bandits with noncompliance should be framed by objective, not by feedback alone. Recommendation welfare, structural treatment learning, and inference are different targets, and they can pull the design in different directions. In finite contexts, BRACE shows that one can obtain parameter-free, weak-ID-safe structural identification while also retaining a clean operational solution when REC is the goal. The empirical study reinforces the same point from the other direction: on easy problems safety appears as regret, under weak identification as abstention and interval width, under homogeneity failure as a reason to prefer REC over TRT, and under extra instruments as tighter but still objective-specific structural information. In rich contexts, the orthogonal score developed here highlights both the opportunity---product-form bias---and the obstacle---stabilizing inverse compliance maps under adaptive sampling.

\appendix

\section{Empirical Appendix}
\label{app:empirical}

This appendix records the broader empirical benchmark behind Section~\ref{sec:empirical}. The goal is not to turn the paper into a benchmark competition, but to make transparent how the objective-first framing behaves across environments that isolate different difficulties.

\subsection{Benchmark Design}

The executable study uses finite-context environments with exact enumeration of recommendation and treatment policies. This makes every reported REC and TRT value exact for the simulated data-generating process, so the comparison is about online learning behavior rather than offline approximation error. The benchmark artifacts use $10$ seeds per scenario and each scenario's default horizon ($2048$ or $4096$). In the main text we display curated slices of those artifacts to keep the core comparisons readable; this appendix shows the broader benchmark views.

The algorithm families are grouped by the target they implicitly privilege:
\begin{enumerate}[leftmargin=2em]
    \item \textbf{REC-first methods:} BRACE-REC, BRACE-REC-FAST, Chosen-UCB, ThompsonBounded, and RECERT's recommendation component.
    \item \textbf{TRT-oriented methods:} BRACE-TRT, BRACE-TRT-FAST, BRACE-TRT-Partial, the 2SLS family, and Actual-UCB when it effectively updates toward treatment labels.
    \item \textbf{Inference methods:} BRACE-INF, BRACE-INF-Partial, and a naive structural interval baseline.
\end{enumerate}

RECERT is especially useful for the interpretation of the paper. It always produces a recommendation policy for current deployment, but treats the structural treatment side as a parallel certification problem. This matches the paper's intended conceptual split: REC is a deployment problem under the current channel, while TRT is a structural claim about a future direct-assignment regime.
Accordingly, the right deployment sanity checks for TRT are the BRACE-TRT and BRACE-TRT-Partial tracks, not RECERT. RECERT is intentionally stricter on the treatment side because its purpose is to keep structural claims diagnostic unless certification is overwhelming.

\subsection{Scenario Catalog}

The benchmark scenarios are purpose-built to separate distinct failure modes or advantages.
\begin{enumerate}[leftmargin=2em]
    \item \textbf{Direct-control equivalence:} recommendations equal treatments, so REC and TRT collapse exactly.
    \item \textbf{Strong IV easy:} a benign one-context setting where all methods eventually learn the right answer and differences are mostly about efficiency.
    \item \textbf{Private-signal operational advantage:} the downstream actor uses private information unavailable to the learner, so the best recommendation policy is strictly better than every direct treatment rule measurable by the learner.
    \item \textbf{Weak IV abstention:} the first stage is nearly singular but the structural gap is large, isolating wide intervals and deliberate non-certification.
    \item \textbf{Weak IV, small structural gap:} the first stage is nearly singular and the structural gap is tiny, so unstable treatment learners can be confidently wrong.
    \item \textbf{Homogeneity violation:} the IV identity fails because treatment contrasts differ by compliance type, isolating the precise assumption behind TRT point identification.
    \item \textbf{Operational-structural tradeoff:} the IV model is valid but the REC-optimal recommendation and TRT-optimal treatment differ, showing that even correct structure does not collapse the two objectives.
    \item \textbf{Workflow redesign / future control:} the current recommendation channel is a bottleneck, so present-patient REC and future-regime TRT point in different directions.
    \item \textbf{Harmful actual-treatment trap:} updating on realized treatment induces a recommendation learner to prefer the wrong action.
    \item \textbf{Rare context:} one low-probability context makes simultaneous safe guarantees operationally expensive.
    \item \textbf{Rectangular overidentification:} three recommendation labels identify two treatments, showing how extra instruments can sharpen structural information.
    \item \textbf{Weak IV rescued by extra instrument:} the square projection is weak, but an additional recommendation arm materially improves structural uncertainty and deployment.
\end{enumerate}

\subsection{Scenario-by-Scenario Findings}

\paragraph{Direct-control equivalence.}
This is the empirical anchor for the historical treatment-first regime. Every policy has exactly the same REC and TRT value. The best policy has value $0.775$ under both objectives, and simple baselines recover it quickly: ThompsonBounded has REC regret $12.34$ and 2SLS-epsilon-decay has TRT regret $8.51$. BRACE-INF still gives coverage $1.0$ with interval width $0.7235$, but BRACE's deployment side remains conservative at this horizon even though the target is structurally easy.

\paragraph{Easy strong IV.}
All methods estimate the correct policy, but the constants reveal what BRACE is paying for. At the scenario's default horizon, BRACE-REC has mean operational regret $167.25$, BRACE-REC-FAST improves this to $65.35$, and Actual-UCB is at $0.95$. On the TRT side, BRACE-TRT deploys correctly but is much slower than the 2SLS family, while BRACE-TRT-FAST remains conservative and abstains. RECERT behaves like a REC-first method here: it keeps deployed REC value $0.75$, estimates the right TRT value, and still withholds structural deployment on this horizon.

\paragraph{Private-signal operational advantage.}
This is the cleanest empirical manifestation of Proposition~\ref{prop:strict}. REC methods achieve value $1.0$, while TRT-oriented methods top out at $0.5$. RECERT mirrors the intended paper message exactly: it deploys REC value $1.0$ and reports a parallel TRT estimate of $0.5$ without structurally deploying it.

\paragraph{Weak IV regimes.}
The large-gap weak-IV environment isolates conservative uncertainty. BRACE-INF maintains coverage $1.0$ with interval width $1.092$ and certified share $0.0$, while unsafe baselines often look deceptively good on regret because the signal eventually points the right way. The small-gap weak-IV environment isolates incorrect structural deployment. BRACE-style TRT methods abstain, whereas unsafe structural learners act and incur nontrivial wrong-action rates: about $0.10$ for Actual-UCB, $0.20$ for adaptive 2SLS, $0.40$ for epsilon-decay 2SLS, $0.60$ for Chosen-UCB, and $0.70$ for fixed-schedule 2SLS. RECERT continues to deploy its recommendation side while refusing a treatment deployment.

\paragraph{Homogeneity violation.}
This is the key empirical meeting point of the conceptual and technical sides of the paper. REC remains strong and deployable: RECERT keeps deployed REC value $0.95$. But TRT should not be trusted as a point-deployment target. Unsafe structural baselines deploy the wrong structural policy with value $0.45$, while the REC-first certificate view abstains. The lesson is not just that misspecification hurts; it is that REC is the target that remains coherent when the homogeneity bridge to TRT breaks.

\paragraph{Objective mismatch under valid IV assumptions.}
In the tradeoff environment, the IV model is valid and yet REC and TRT still disagree. RECERT deploys REC value $0.85$ while estimating TRT value $0.9$. This is important because it shows that the REC/TRT split is not merely a misspecification story; it persists even in the paper's idealized structural regime.

\paragraph{Workflow redesign / future control.}
This is the cleanest empirical case for stakeholders who care about future channel reform. The best current-channel recommendation has REC value $0.69$, but the best direct-treatment rule has TRT value $0.90$. RECERT therefore deploys the current best recommendation while estimating the stronger structural target in parallel. Unlike the private-signal case, TRT is not conceptually misaligned here; it is the right target for future patients if the workflow is redesigned. BRACE-INF covers the structural target with width $0.4595$, and the base BRACE-TRT rule deploys correctly in $90\%$ of runs.

\paragraph{Harmful actual-treatment trap.}
This environment sharpens the warning against naively treating realized treatment as the learning target. Actual-UCB and Comply-UCB deploy REC value $0.66$ instead of the optimum $0.81$, while ThompsonBounded, Chosen-UCB, and BRACE-REC-FAST recover the right recommendation. On the TRT side, adaptive 2SLS and Chosen-UCB recover the correct structural target, but epsilon-decay and fixed-schedule 2SLS misfire in $10\%$ of runs.

\paragraph{Rectangular and rescued-IV settings.}
The rectangular environments show why the square specialization should be read as a theorem setting rather than a universal modeling principle. In the overidentified case, BRACE-INF has mean interval width $0.3605$, while BRACE-INF-Partial shrinks this to $0.1768$. In the rescued-IV case, the square weak-ID picture remains at width $1.0656$, but rectangular partial-identification again tightens uncertainty to $0.1769$ and enables BRACE-TRT-Partial to deploy the correct structural policy where square point-ID BRACE still abstains. These are genuine information gains from the design, not just cosmetic extensions of the notation.

\paragraph{Rare context.}
All methods recover the right final policy, but the exploration cost of simultaneous safety becomes obvious. BRACE-REC-FAST and BRACE-TRT-FAST still abstain at the default horizon and incur regret $272.6$, while simpler baselines such as Actual-UCB and 2SLS-epsilon-decay solve the problem with regret below $10$. This is the cleanest benchmark reminder that safe contextwise guarantees can be operationally expensive even when the final answer is simple.

\subsection{Additional Figures}

\begin{figure}[p]
\centering
\begin{minipage}[t]{0.49\textwidth}
    \centering
    \includegraphics[width=\linewidth]{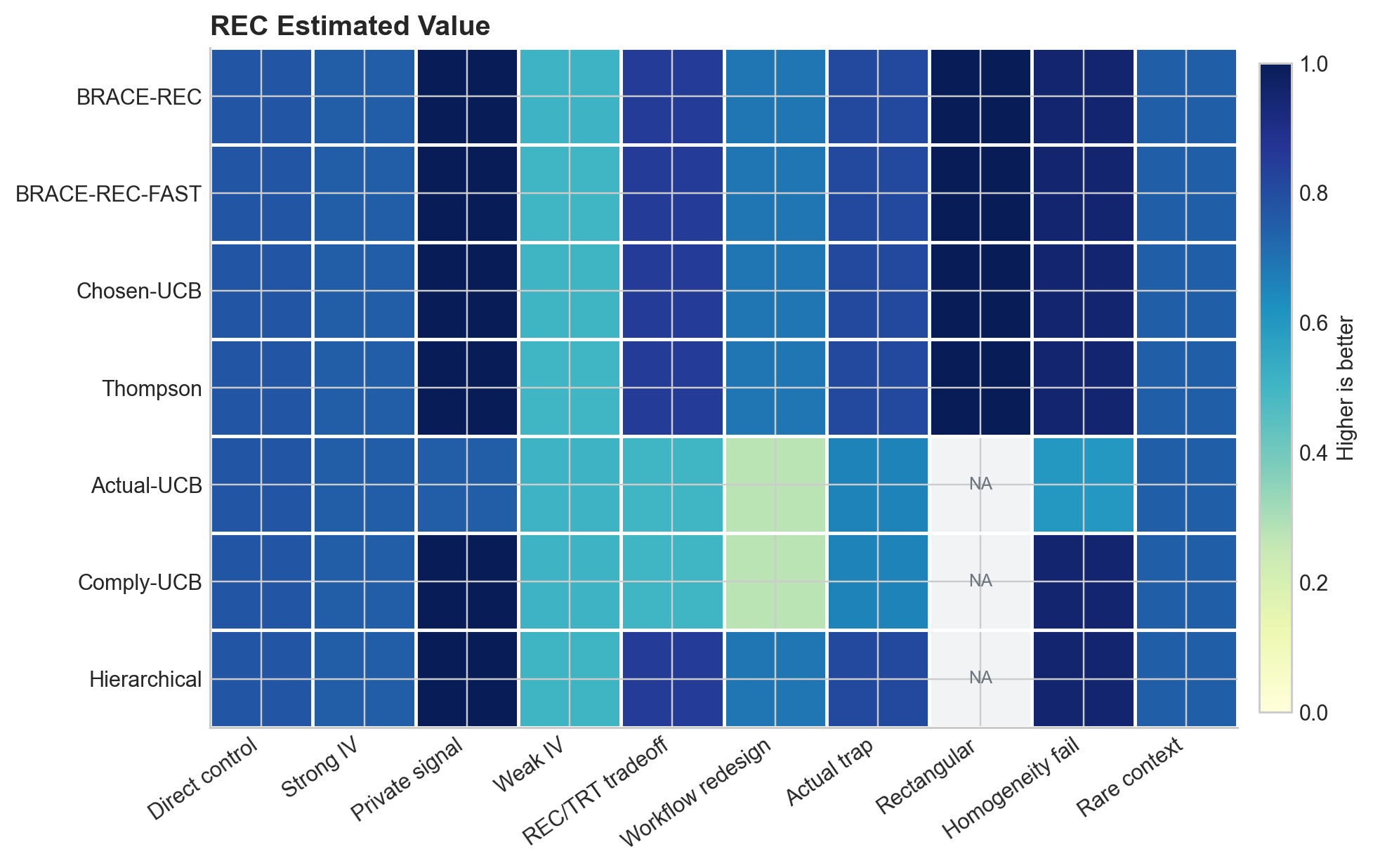}
\end{minipage}
\hfill
\begin{minipage}[t]{0.49\textwidth}
    \centering
    \includegraphics[width=\linewidth]{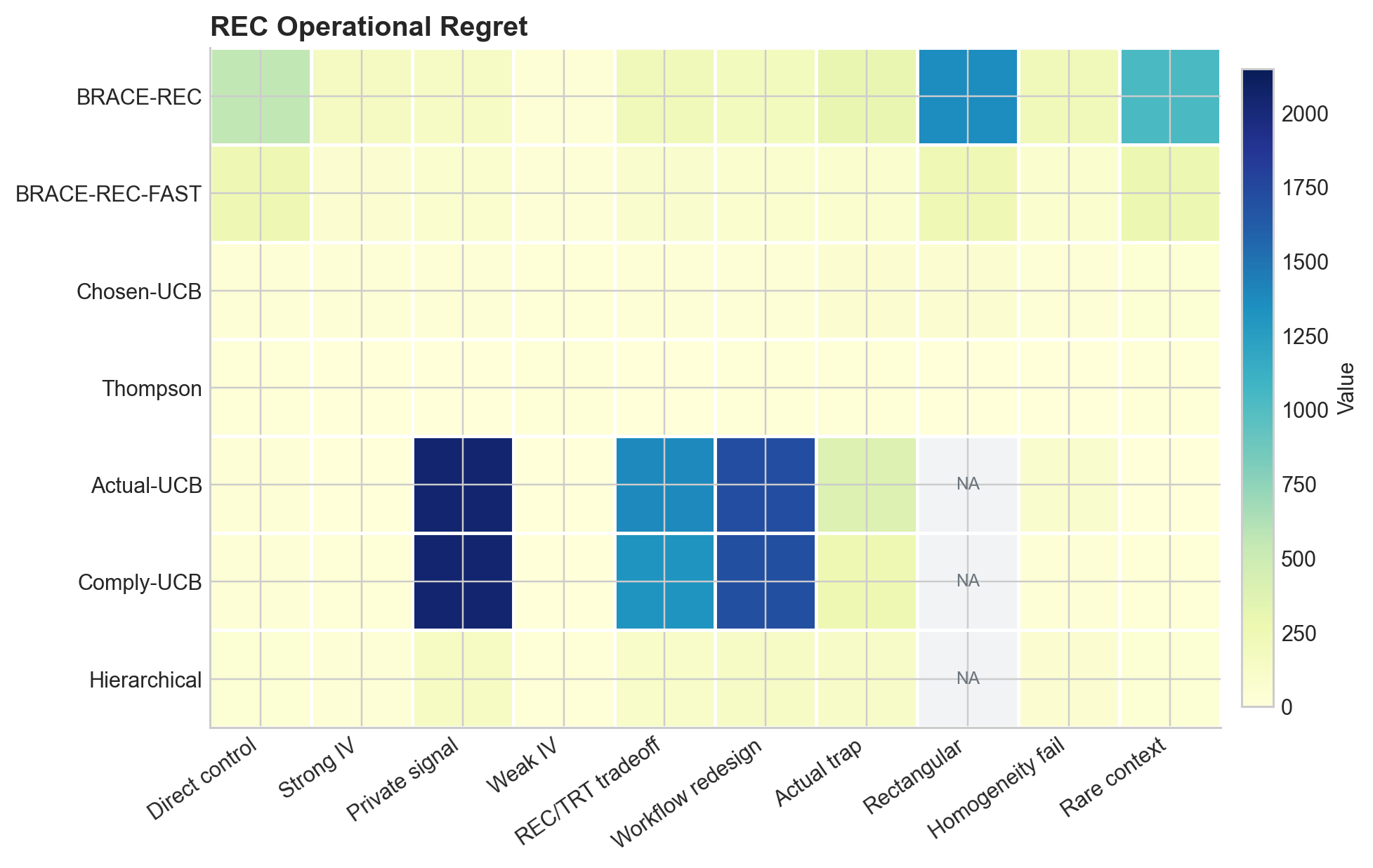}
\end{minipage}
\caption{Recommendation-track summary plots. Left: mean estimated REC value. Right: mean operational regret. Together they show that many methods agree on the final REC answer in easy environments, while differing sharply in how much operational welfare they spend during learning.}
\end{figure}

\begin{figure}[p]
\centering
\begin{minipage}[t]{0.49\textwidth}
    \centering
    \includegraphics[width=\linewidth]{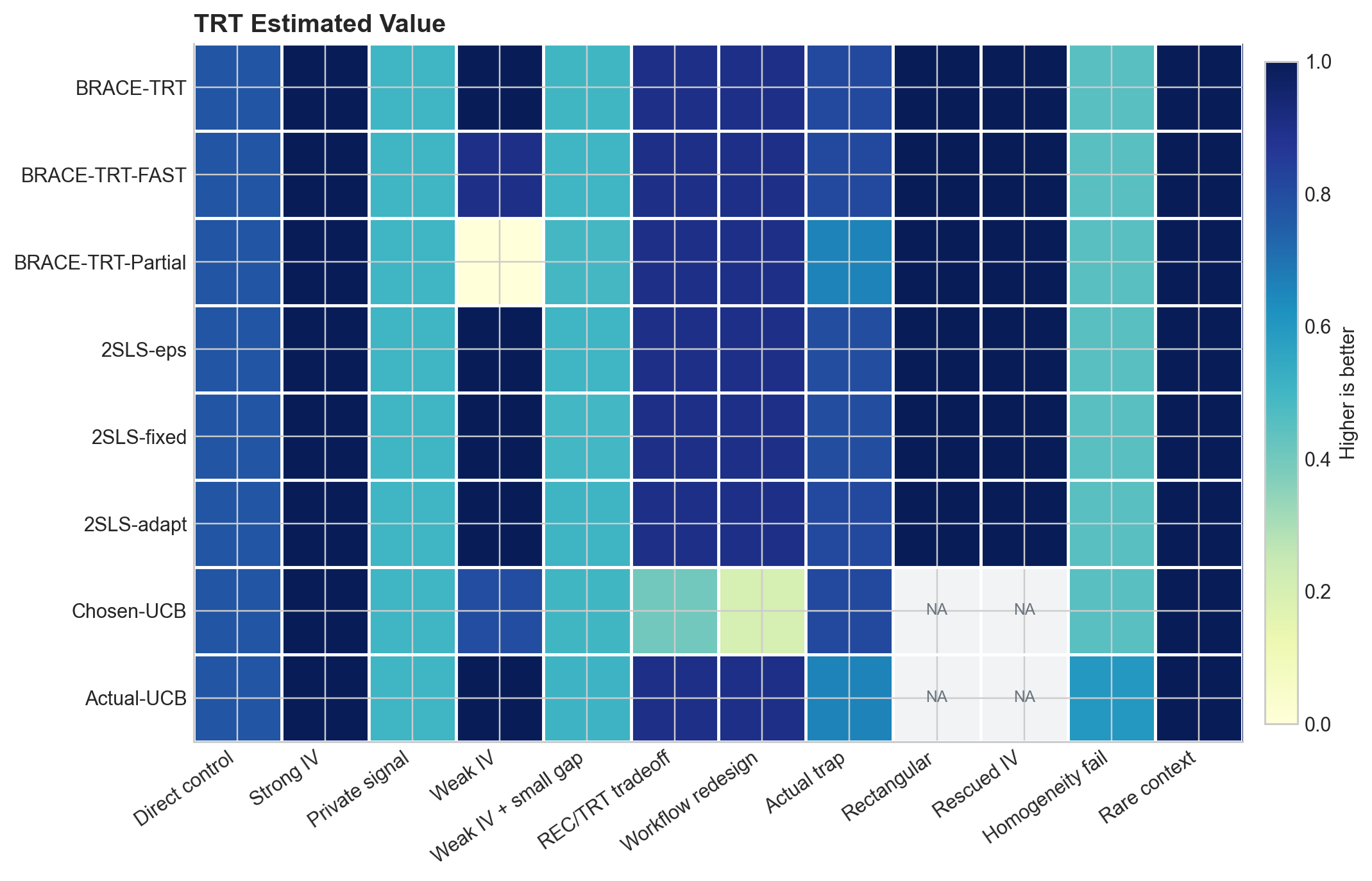}
\end{minipage}
\hfill
\begin{minipage}[t]{0.49\textwidth}
    \centering
    \includegraphics[width=\linewidth]{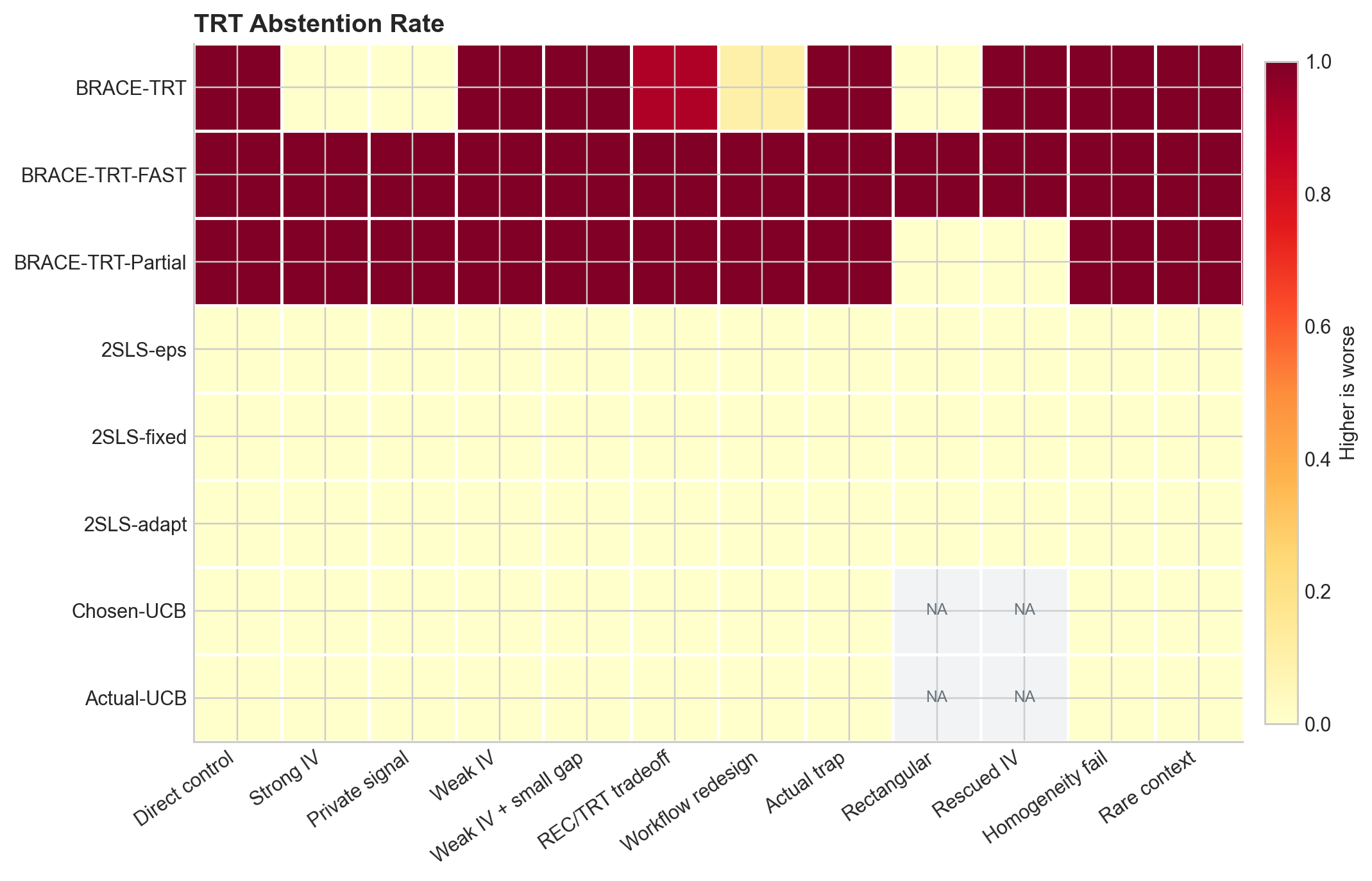}
\end{minipage}
\caption{Treatment-track summary plots. Left: mean estimated TRT value. Right: abstention rate. The combination makes the main BRACE tradeoff visible: structural caution can preserve validity but often appears as delayed or missing deployment.}
\end{figure}

\begin{figure}[p]
\centering
\begin{minipage}[t]{0.49\textwidth}
    \centering
    \includegraphics[width=\linewidth]{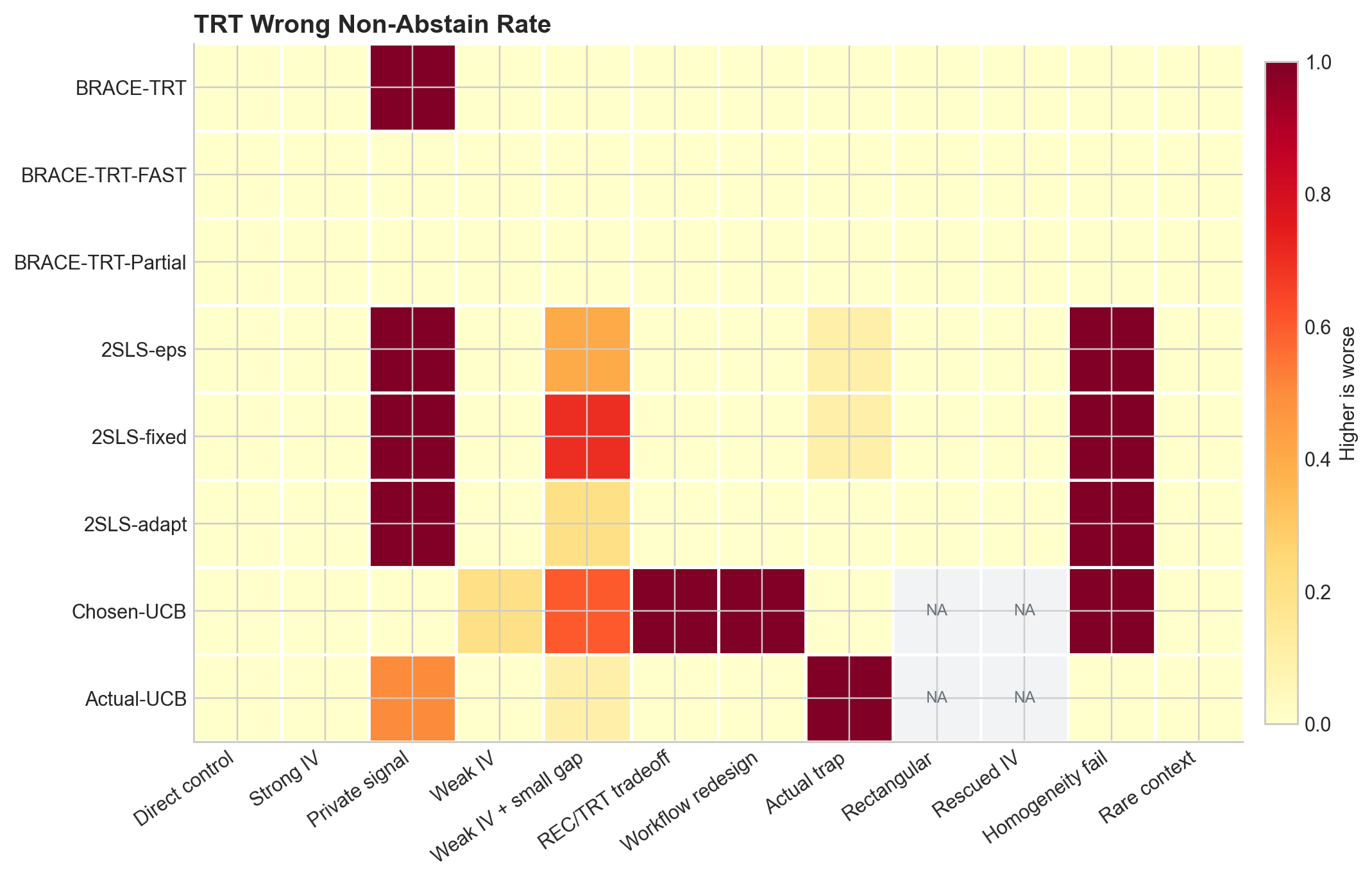}
\end{minipage}
\hfill
\begin{minipage}[t]{0.49\textwidth}
    \centering
    \includegraphics[width=\linewidth]{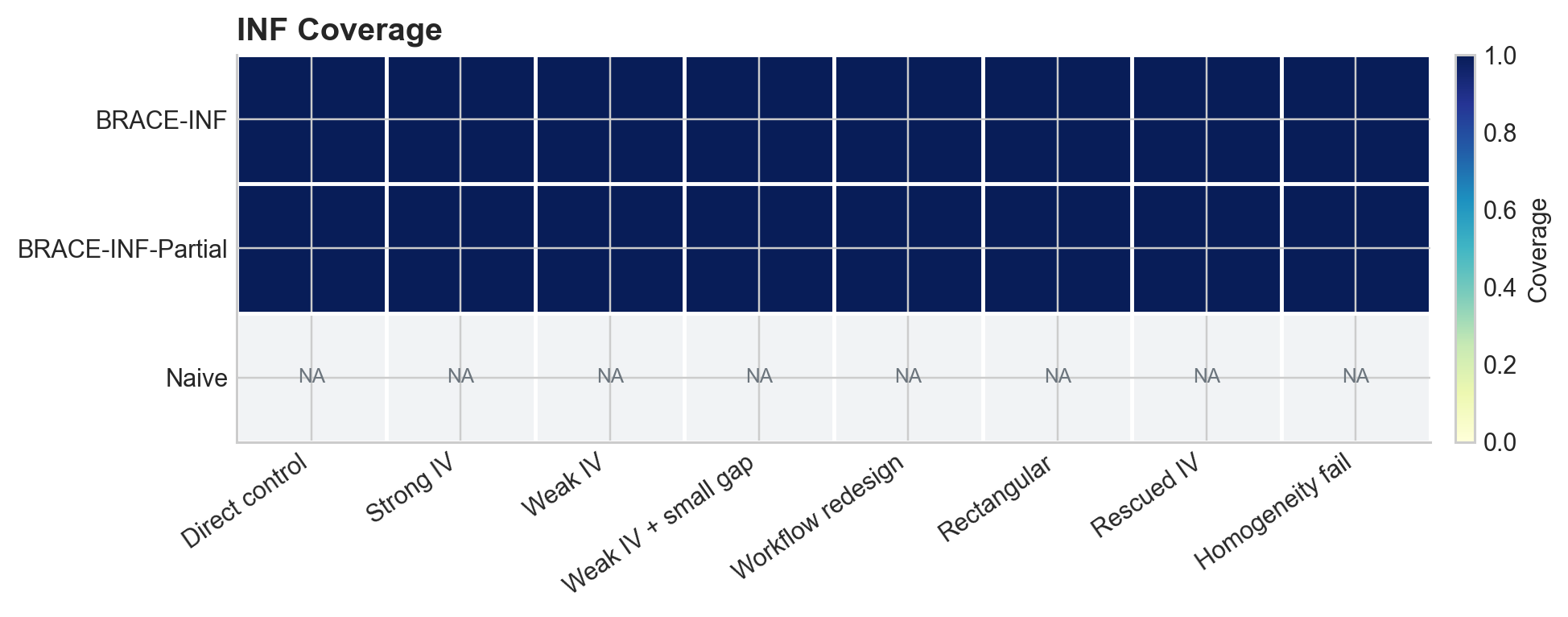}
\end{minipage}
\caption{Structural risk and uncertainty validity. Left: wrong-non-abstain rate for TRT algorithms. Right: coverage of the best-policy TRT value in the inference track. The left panel highlights where unsafe methods make incorrect structural decisions; the right panel shows that the interval procedures remain conservative in the benchmark.}
\end{figure}

\begin{figure}[p]
\centering
\begin{minipage}[t]{0.49\textwidth}
    \centering
    \includegraphics[width=\linewidth]{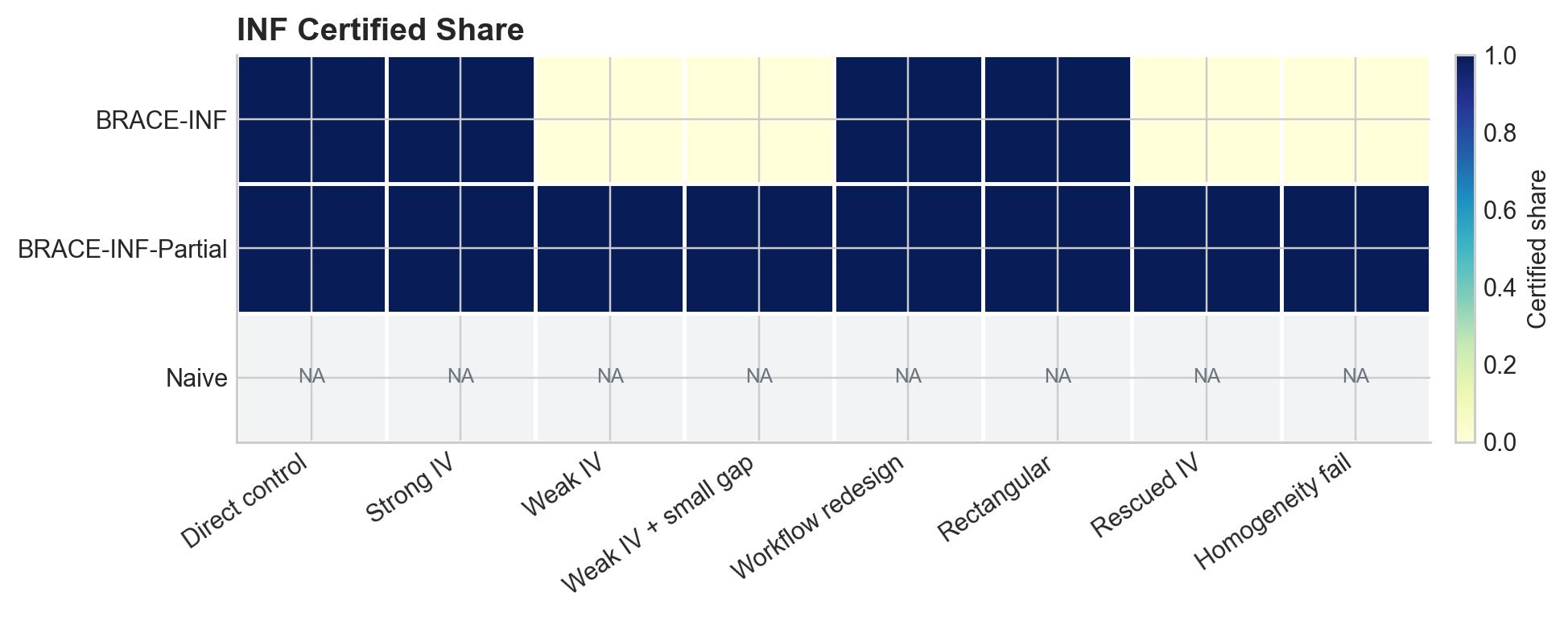}
\end{minipage}
\hfill
\begin{minipage}[t]{0.49\textwidth}
    \centering
    \includegraphics[width=\linewidth]{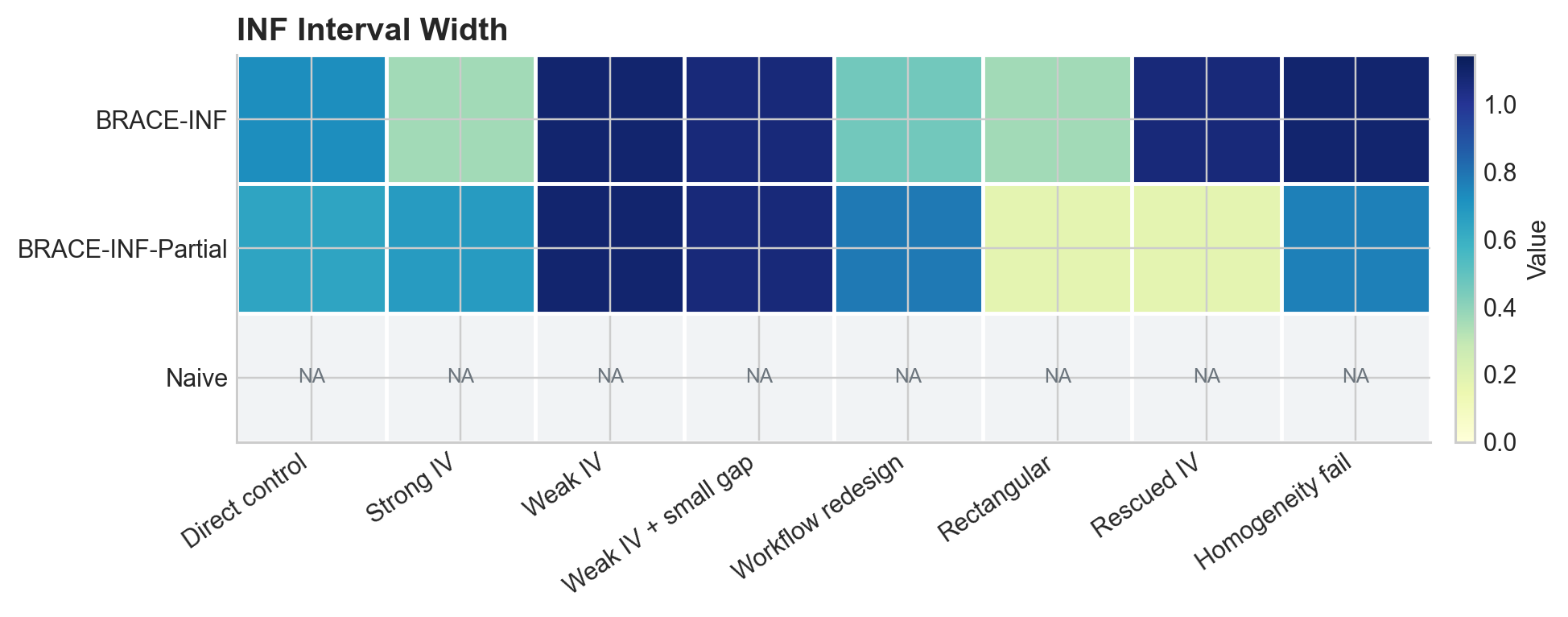}
\end{minipage}
\caption{Inference diagnostics. Certified-share and interval-width plots together show how the benchmark separates strong-ID, weak-ID, and overidentified environments. Wide intervals in weak-ID cases are a feature of the safety story, not a bug in the implementation.}
\end{figure}

\begin{figure}[p]
\centering
\begin{minipage}[t]{0.49\textwidth}
    \centering
    \includegraphics[width=\linewidth]{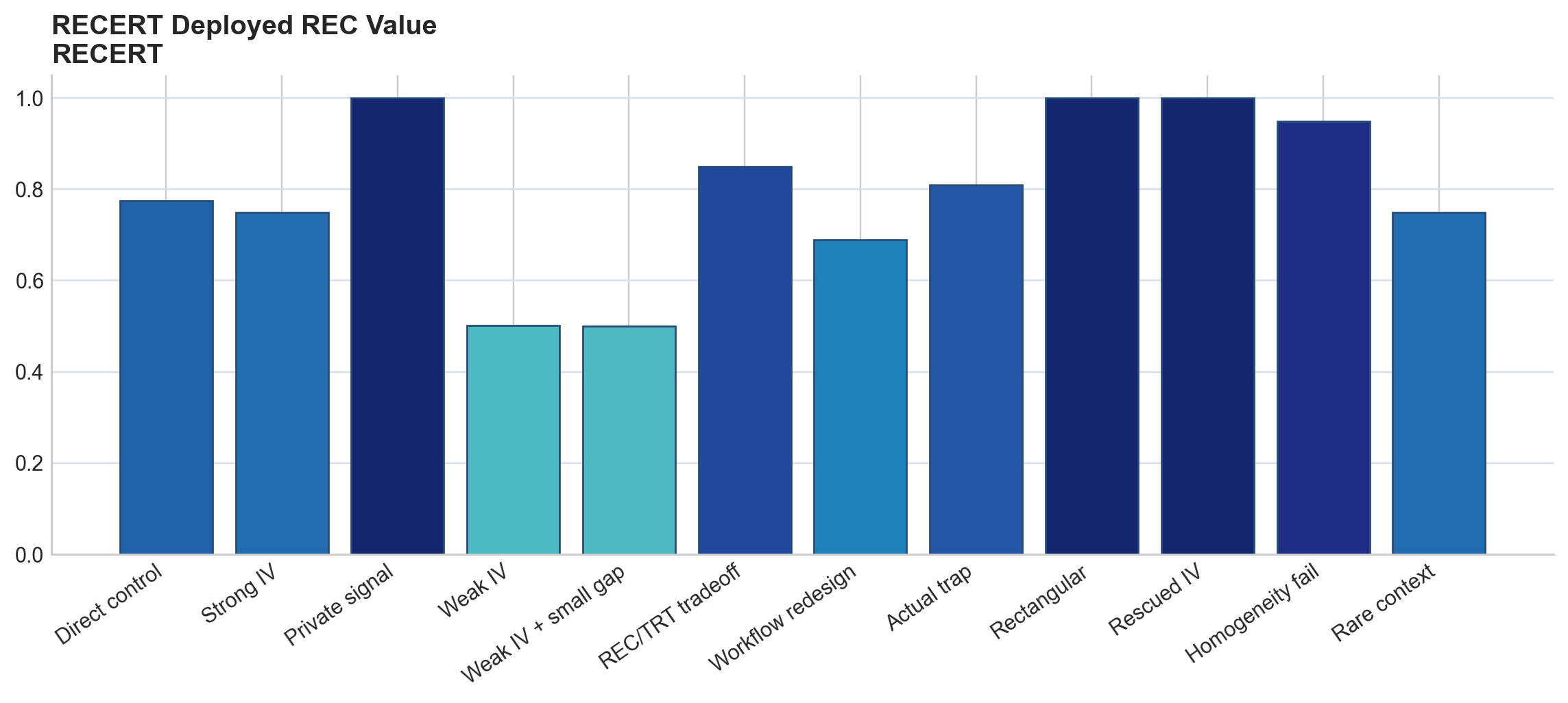}
\end{minipage}
\hfill
\begin{minipage}[t]{0.49\textwidth}
    \centering
    \includegraphics[width=\linewidth]{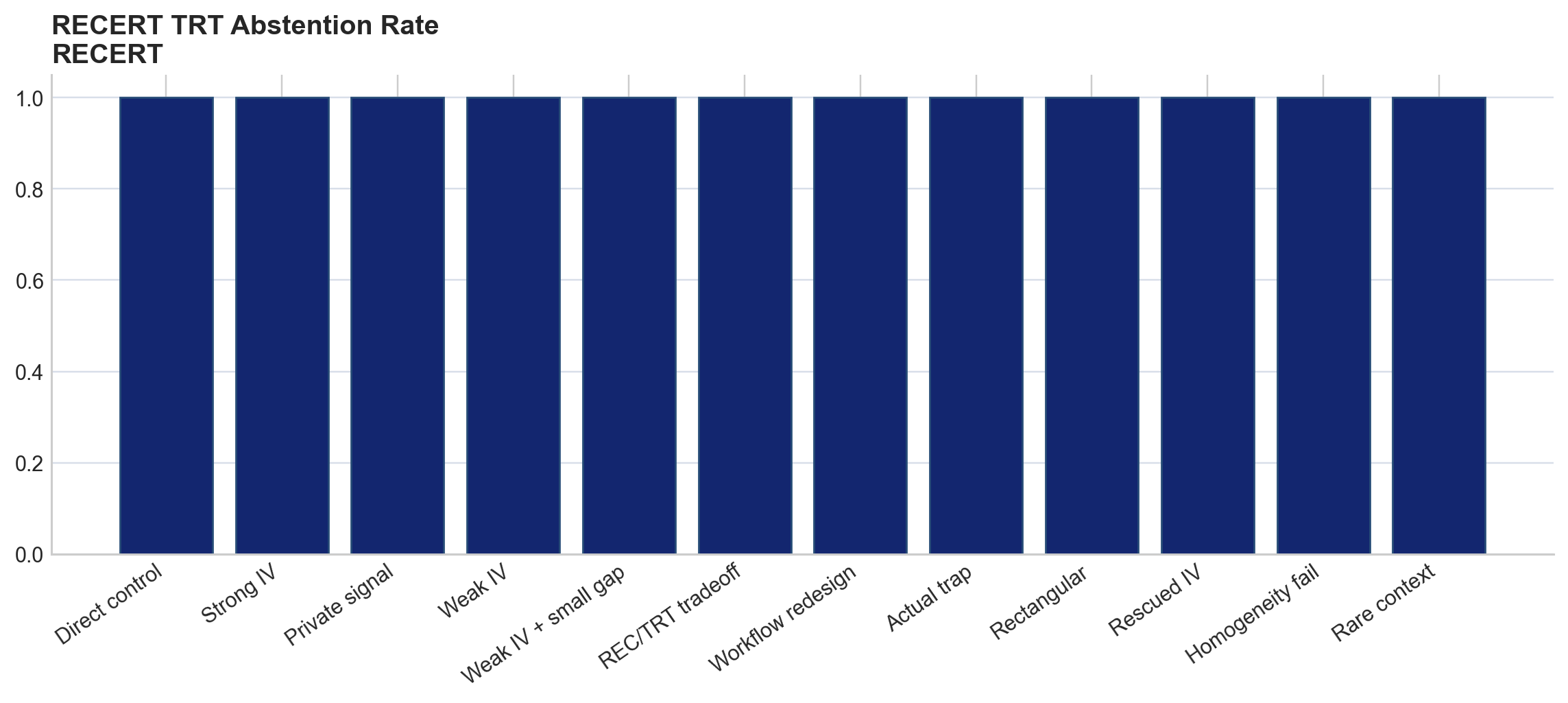}
\end{minipage}
\caption{RECERT behavior. Left: deployed REC value for the joint extension. Right: TRT abstention rate for the same method. This is the compact graphical version of the paper's intended operational stance: strong recommendation deployment together with structural caution.}
\end{figure}

\begin{figure}[p]
\centering
\includegraphics[width=0.72\textwidth]{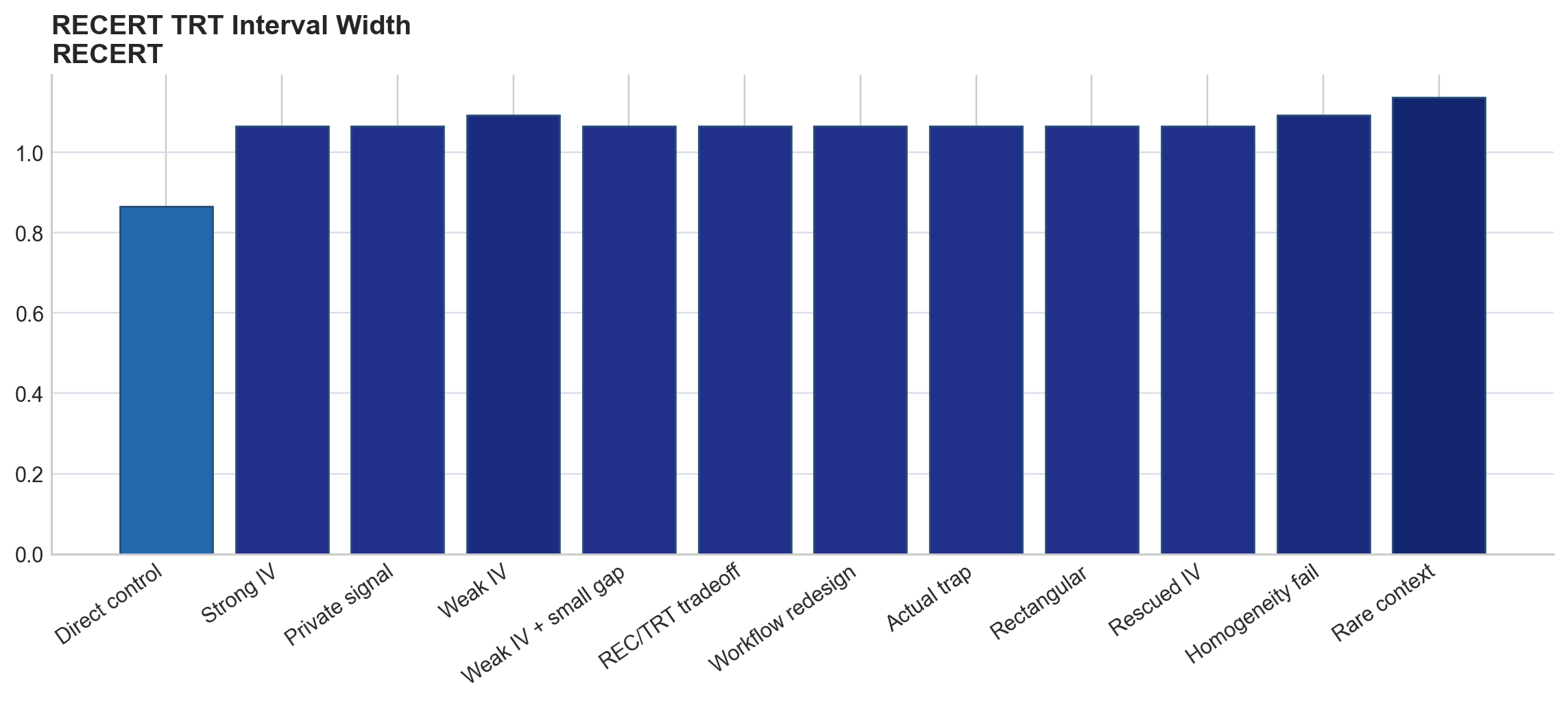}
\caption{TRT interval width for RECERT. The joint method is operationally useful precisely because its treatment side is interpreted as a certificate and diagnostic, not as a compulsory deployment rule.}
\end{figure}


\begin{thebibliography}{99}

\bibitem{athey2022}
Susan Athey, Undral Byambadalai, Vitor Hadad, Sanath Kumar Krishnamurthy, Weiwen Leung, and Joseph Jay Williams.
\newblock Contextual bandits in a survey experiment on charitable giving: Within-experiment outcomes versus policy learning.
\newblock \emph{arXiv preprint arXiv:2211.12004}, 2022.

\bibitem{angrist1996}
Joshua D. Angrist, Guido W. Imbens, and Donald B. Rubin.
\newblock Identification of causal effects using instrumental variables.
\newblock \emph{Journal of the American Statistical Association}, 91(434):444--455, 1996.

\bibitem{bloom1984}
Howard S. Bloom.
\newblock Accounting for no-shows in experimental evaluation designs.
\newblock \emph{Evaluation Review}, 8(2):225--246, 1984.

\bibitem{chen2023}
Siyu Chen, Yitan Wang, Zhaoran Wang, and Zhuoran Yang.
\newblock A unified framework of policy learning for contextual bandit with confounding bias and missing observations.
\newblock \emph{arXiv preprint arXiv:2303.11187}, 2023.

\bibitem{cook2024}
Thomas Cook, Alan Mishler, and Aaditya Ramdas.
\newblock Semiparametric efficient inference in adaptive experiments.
\newblock In \emph{Proceedings of the Third Conference on Causal Learning and Reasoning}, volume 236 of \emph{Proceedings of Machine Learning Research}, pages 1033--1064, 2024.

\bibitem{dalal2024}
Abhinandan Dalal, Patrick Bl\"obaum, Shiva Kasiviswanathan, and Aaditya Ramdas.
\newblock Anytime-valid inference for double/debiased machine learning of causal parameters.
\newblock \emph{arXiv preprint arXiv:2408.09598}, 2024.

\bibitem{ich1998}
International Council for Harmonisation.
\newblock E9 Statistical Principles for Clinical Trials.
\newblock Guidance for Industry, 1998.

\bibitem{ich2021}
International Council for Harmonisation.
\newblock E9(R1) Addendum on Estimands and Sensitivity Analysis in Clinical Trials to the Guideline on Statistical Principles for Clinical Trials.
\newblock Guidance for Industry, 2021.

\bibitem{imbens1994}
Guido W. Imbens and Joshua D. Angrist.
\newblock Identification and estimation of local average treatment effects.
\newblock \emph{Econometrica}, 62(2):467--475, 1994.

\bibitem{dellapenna2016}
Nicol\'as Della Penna, Mark D. Reid, and David Balduzzi.
\newblock Compliance-aware bandits.
\newblock \emph{arXiv preprint arXiv:1602.02852}, 2016.

\bibitem{kallus2018}
Nathan Kallus.
\newblock Instrument-armed bandits.
\newblock In \emph{Proceedings of Algorithmic Learning Theory}, volume 83 of \emph{Proceedings of Machine Learning Research}, pages 529--546, 2018.

\bibitem{kato2020}
Masahiro Kato, Takuya Ishihara, Junya Honda, and Yusuke Narita.
\newblock Efficient adaptive experimental design for average treatment effect estimation.
\newblock \emph{arXiv preprint arXiv:2002.05308}, 2020.

\bibitem{kato2021}
Masahiro Kato, Kenichiro McAlinn, and Shota Yasui.
\newblock The adaptive doubly robust estimator and a paradox concerning logging policy.
\newblock In \emph{Advances in Neural Information Processing Systems 34}, 2021.

\bibitem{karampatziakis2021}
Nikos Karampatziakis, Paul Mineiro, and Aaditya Ramdas.
\newblock Off-policy confidence sequences.
\newblock In \emph{Proceedings of the 38th International Conference on Machine Learning}, volume 139 of \emph{Proceedings of Machine Learning Research}, pages 5301--5310, 2021.

\bibitem{kveton2023}
Branislav Kveton, Yi Liu, Johan Matteo Kruijssen, and Yisu Nie.
\newblock Non-compliant bandits.
\newblock In \emph{Proceedings of the 32nd ACM International Conference on Information and Knowledge Management}, pages 1138--1147, 2023.

\bibitem{liang2023}
Biyonka Liang and Iavor Bojinov.
\newblock An experimental design for anytime-valid causal inference on multi-armed bandits.
\newblock \emph{arXiv preprint arXiv:2311.05794}, 2023.

\bibitem{ngo2021}
Dung Daniel T. Ngo, Logan Stapleton, Vasilis Syrgkanis, and Zhiwei Steven Wu.
\newblock Incentivizing compliance with algorithmic instruments.
\newblock In \emph{Proceedings of the 38th International Conference on Machine Learning}, volume 139 of \emph{Proceedings of Machine Learning Research}, pages 8045--8055, 2021.

\bibitem{oprescu2024}
Miruna Oprescu and Nathan Kallus.
\newblock Estimating heterogeneous treatment effects by combining weak instruments and observational data.
\newblock In \emph{Advances in Neural Information Processing Systems 37}, 2024.
\newblock Also available as \emph{arXiv preprint arXiv:2406.06452}.

\bibitem{oprescu2025}
Miruna Oprescu, Brian M. Cho, and Nathan Kallus.
\newblock Efficient adaptive experimentation with noncompliance.
\newblock \emph{arXiv preprint arXiv:2505.17468}, 2025.

\bibitem{shao2024}
Daqian Shao, Ashkan Soleymani, Francesco Quinzan, and Marta Kwiatkowska.
\newblock Learning decision policies with instrumental variables through double machine learning.
\newblock In \emph{Proceedings of the 41st International Conference on Machine Learning}, 2024.
\newblock Also available as \emph{arXiv preprint arXiv:2405.08498}.

\bibitem{sommer1991}
Alfred Sommer and Scott L. Zeger.
\newblock On estimating efficacy from clinical trials.
\newblock \emph{Statistics in Medicine}, 10(1):45--52, 1991.

\bibitem{waudby2024}
Ian Waudby-Smith, Lili Wu, Aaditya Ramdas, Nikos Karampatziakis, and Paul Mineiro.
\newblock Anytime-valid off-policy inference for contextual bandits.
\newblock \emph{ACM/IMS Journal of Data Science}, 1(3):1--42, 2024.

\bibitem{zelen1979}
Marvin Zelen.
\newblock A new design for randomized clinical trials.
\newblock \emph{New England Journal of Medicine}, 300(22):1242--1245, 1979.

\end{thebibliography}
\end{document}